\documentclass[10pt]{article} 

\usepackage{amssymb}            
\usepackage{mathtools}          
\usepackage{mathrsfs}           
\mathtoolsset{showonlyrefs}     
\usepackage{graphicx}           
\usepackage{subcaption}         
\usepackage[space]{grffile}     
\usepackage{url}                

\usepackage[round]{natbib}

\usepackage[bookmarks=false,colorlinks=true,linkcolor=blue,urlcolor=magenta,citecolor=red,pdfstartview=FitH,pagebackref]{hyperref}
\usepackage{url}

\usepackage{graphicx} 
\usepackage{amsthm}
\usepackage{amsmath}
\usepackage{amsfonts,amssymb}
\usepackage{mathtools}
\usepackage{xcolor}
\usepackage{booktabs}
\usepackage{makecell}
\usepackage{enumitem}
\usepackage{tikz}
\usepackage{algpseudocode}
\usepackage{multirow}
\usepackage{comment}
\usepackage{algorithm}

\newcommand{\norm}[1]{\left\lVert#1\right\rVert}

\newcommand{\PID}{g}
\newcommand{\Lp}{L_\PID^\pi}
\newcommand{\Ls}{L_\PID^*}
\newcommand{\Ap}{A_\PID^\pi}

\newcommand{\Bp}{b_\PID^\pi}

\newcommand{\BRstar}{\text{BR}^{*}}
\newcommand{\BRpi}{\text{BR}^{\pi}}
\newcommand{\BRhat}{\widehat{\text{BR}}}

\newcommand{\OO}{{\mathcal{O}}}
\newcommand{\DD}{{\mathcal{D}}}

\newcommand{\fullQ}{\tilde{Q}}
\newcommand{\fullV}{\tilde{V}}

\newcommand{\XX}{{\mathcal{X}}}
\newcommand{\GG}{{\mathcal{G}}}
\newcommand{\FSet}{{ \{\texttt{v, v', z}\} }}
\renewcommand{\AA}{{\mathcal{A}}}
\def\reals{\mathbb{R}} 
\def\defeq{\triangleq} 

\newcommand{\inner}[1]{\left\langle #1 \right\rangle}

\newtheorem{theorem}{Theorem}
\newtheorem{lemma}{Lemma}

\newtheorem{assumption}{Assumption}
\newtheorem{proposition}{Proposition}
\newtheorem{definition}{Definition}

\usepackage{amsmath,amsfonts,bm}

\def\1{\bm{1}}

\def\eps{{\epsilon}}

\DeclareMathAlphabet{\mathsfit}{\encodingdefault}{\sfdefault}{m}{sl}
\SetMathAlphabet{\mathsfit}{bold}{\encodingdefault}{\sfdefault}{bx}{n}

\newcommand{\E}{\mathbb{E}}

\newcommand{\Var}{\mathrm{Var}}

\DeclareMathOperator*{\argmax}{arg\,max}

\usepackage[accepted]{rlc}

\newcommand{\Vpi}{V^\pi}
\newcommand{\Vtd}{V^\mathrm{TD}}
\newcommand{\Vpid}{V^\mathrm{PID}}
\newcommand{\ctd}{c^\mathrm{TD}}
\newcommand{\cpid}{c^\mathrm{PID}}
\newcommand{\deltatd}{\delta^\mathrm{TD}}
\newcommand{\deltapid}{\delta^\mathrm{PID}}
\newcommand{\Ctd}{C^\mathrm{TD}}
\newcommand{\Cpid}{C^\mathrm{PID}}
\newcommand{\Btd}{B^\mathrm{TD}}
\newcommand{\Bpid}{B^\mathrm{PID}}
\newcommand{\ltd}{{l^\mathrm{TD}}}
\newcommand{\lpid}{{l^\mathrm{PID}}}

\newcommand{\Estattd}{{E_\mathrm{stat}^\mathrm{TD}}}
\newcommand{\Estatpid}{{E_\mathrm{stat}^\mathrm{PID}}}

\newcommand{\Eopttd}{{E_\mathrm{opt}^\mathrm{TD}}}
\newcommand{\Eoptpid}{{E_\mathrm{opt}^\mathrm{PID}}}

\newcommand{\Qopt}{Q^*}

\newcommand{\fullVpi}{\tilde{V}^\pi}
\newcommand{\fullQopt}{\tilde{Q}^*}
\newcommand{\snorm}[1]{\norm{#1}_{2,S}}

\newcommand{\Tpi}{T^\pi}
\newcommand{\Topt}{{T^\star}}

\newcommand{\piopt}{{\pi^*}}

\newcommand{\PKernel}{\mathcal{P}}

\newcommand{\PKernelpi}{{\PKernel}^{\pi}}

\newcommand{\RKernel}{\mathcal{R}}

\newcommand{\rpi}{r^\pi}

\newcommand{\MM}{\mathcal{M}}

\newcommand{\Real}{\mathbb R}

\newcommand{\XA}{\XX\times\AA}

\newcommand{\BB}{\mathcal{B}}

\newcommand{\tf}{\texttt{f}}
\newcommand{\tv}{\texttt{v}}
\newcommand{\tz}{\texttt{z}}
\newcommand{\tvp}{\texttt{v'}}

\title{PID Accelerated Temporal Difference Algorithms}

\author{Mark Bedaywi\thanks{These authors contributed equally to this work.} \quad Amin Rakhsha\footnotemark[1] \quad Amir-massoud Farahmand \\
Department of Computer Science, University of Toronto
\\
Vector Institute}

\begin{document}

\maketitle

\begin{abstract}
	Long-horizon tasks, which have a large discount factor, pose a challenge for most conventional reinforcement learning (RL) algorithms. Algorithms such as Value Iteration and Temporal Difference (TD) learning have a slow convergence rate and become inefficient in these tasks. When the transition distributions are given, PID~VI was recently introduced to accelerate the convergence of Value Iteration using ideas from control theory. Inspired by this, we introduce PID TD Learning and PID Q-Learning algorithms for the RL setting, in which only samples from the environment are available. 
 We give a theoretical analysis of the convergence of PID TD Learning and its acceleration compared to the conventional TD Learning. 
 We also introduce a method for adapting PID gains in the presence of noise and empirically verify its effectiveness.
\end{abstract}

\section{Introduction}
\label{sec:introduction}

The Value Iteration (VI) algorithm is one of the primary dynamic programming methods for solving (discounted) Markov Decision Processes (MDP). It is the foundation of many Reinforcement Learning (RL) algorithms such as the Temporal Difference (TD) Learning~\citep{Sutton1988,TsitsiklisVanRoy97}, Q-Learning~\citep{Watkins1989}, Approximate/Fitted Value Iteration~\citep{Gordon1995,Ernst05,Munos08JMLR,TosatooPirottaDEramoRestelli2017}, and DQN~\citep{mnih2015human,VanHasseltGuezSilver2016}, which can all be seen as sample-based variants of VI.
A weakness of the VI algorithm and the RL algorithms built on top of it is their slow convergence in problems with discount factor $\gamma$ close to $1$, which corresponds to the long-horizon problems where the agent aims to maximize its cumulative rewards far in the future. 
One can show that the error of the value function calculated by VI at iteration $k$ goes to zero with the slow rate of $\OO(\gamma^k)$. The slow convergence rate when $\gamma \approx 1$ also appears in the error analysis of the downstream temporal difference~\citep{Szepesvari1997,EvenDarMansour2003,Wainwright2019} and fitted value iteration algorithms~\citep{Munos08JMLR,FarahmandMunosSzepesvari10,ChenJiang2019,FanWangXieYang2019}.
If $\gamma \approx 1$, these algorithms become very slow and inefficient.
This work introduces accelerated temporal difference learning algorithms that can mitigate this issue.

\citet{farahmand2021pid} recently suggested that one may view the iterates of VI as a dynamical system. This opens up the possibility of using tools from control theory to modify, and perhaps accelerate, the VI's dynamics.
They specifically used the simple class of Proportional-Integral-Derivative (PID) controllers to modify VI, resulting in a new procedure called the PID VI algorithm.
They showed that with a careful choice of the controller gains, PID VI can converge significantly faster than the conventional VI. They also introduced a gain adaptation mechanism, a meta-learning procedure, to automatically choose these gains.

PID VI, similar to VI, is a dynamic programming  algorithm and requires access to the full transition dynamics of the environment. In the RL setting, however, the transition dynamics is not directly accessible to the agent; the agent can only acquire samples from the transition dynamics by interacting with the environment. 


In this work, we show how the ideas of the PID VI algorithm can be used in the RL setting. Our contributions are:
\begin{itemize}
    \item Introduce PID TD Learning and PID Q-Learning algorithms (Section~\ref{sec:pidtd-pidql}) for the RL setting that show accelerated convergence compared to their conventional counterparts.
    \item Theoretically show the convergence and acceleration of the PID TD Learning (Section~\ref{sec:pid-theory}).
    \item A sample-based gain adaptation mechanism to automatically tune the controller gains, reducing the hyperparameter-tuning required for the algorithms (Section~\ref{sec:gain-adaptation}).
\end{itemize}

The new algorithms are a step towards RL algorithms that can tackle long-horizon tasks more efficiently.

\section{Background}
\label{sec:background}

Given a set $\Omega$, let $\MM(\Omega)$ be the set of probability distributions over $\Omega$, and  $\BB(\Omega)$ be the set of bounded functions over $\Omega$.
We consider a discounted MDP~\citep{bertsekas1996neuro,SzepesvariBook10,sutton2018reinforcement} defined as $(\XX , \AA , \PKernel, \RKernel, \gamma)$ where $\XX$ is the finite set of $n$ states, $\AA$ is the finite set of $m$ actions, $\PKernel\colon \XX \times \AA \to \MM(\XX)$ is the transition kernel, $\RKernel: \XX \times \AA \to \MM([0,1])$ is the reward function, and $\gamma \in [0, 1)$ is the discount factor. 

A policy $\pi$ is a function $\pi\colon \mathcal{X} \to \MM(\mathcal{A})$ representing the distribution over the actions an agent would take from each state. Given a policy $\pi$, the functions $\Vpi: \XX \to \reals$ and $Q^\pi \colon \XX \times \AA \to \reals$ are the corresponding \mbox{(state-)value} and action-value functions defined as the expected discounted return when following $\pi$ starting at a certain state or state-action pair. We also let $\PKernelpi \colon \XX \to \MM(\XX)$  and $\RKernel^\pi \colon \XX \to \MM([0,1])$ be the associated transition and reward kernels of policy $\pi$, and $\rpi : \XX \to [0,1] $ be the expected reward of following $\pi$ at any state.

The Policy Evaluation (PE) problem is the problem of finding the value function $\Vpi$ corresponding to a given policy $\pi$ and the Control problem is the problem of finding the policy $\piopt$ that maximizes the corresponding value function $Q^*(x,a) \triangleq Q^{\pi^*}(x,a) = \max_\pi Q^\pi(x,a)$, for each state $x$ and action $a$. 
We shall use $V$ whenever we talk about the PE problem and $Q$ for the Control problem, for the brevity of the presentation.

The Bellman operator, $\Tpi$, and the Bellman optimality operator, $\Topt $, are defined as follows:
\begin{align*}
    &(\Tpi V)(x) \triangleq r^\pi(x) + \gamma\int \mathcal{P}^\pi(\text{d}y\mid x)V(y), &&(\forall x \in \XX),\\
    &(\Topt Q)(x, a) \triangleq r(x, a) + \gamma\int \mathcal{P}(\text{d}y\mid x, a)\max_{a' \in \mathcal{A}}Q(y, a') && (\forall x \in \XX, a \in \AA).
\end{align*}
The Bellman residual operators are defined as $\BRpi V \triangleq \Tpi  V - V$ (for PE) and $\BRstar Q \triangleq \Topt Q - Q$ (for Control). 
The value function $\Vpi$ is the unique function with $\BRpi \Vpi = 0$ and $\Qopt$ is the unique function with $\BRstar \Qopt = 0$.

The iteration $V_{k+1} \gets \Tpi V_k$ converges to $\Vpi$, and the iteration $Q_{k+1} \gets \Topt Q_k$ converges to $\Qopt$. This is known as the VI algorithm.
The convergence is due to the $\gamma$-contraction of the Bellman operators with respect to the supremum norm, and can be proven using  the Banach fixed-point theorem. The result also shows that the convergence rate of VI is $\mathcal{O}(\gamma^k)$. This can be extremely slow for long horizon tasks with $\gamma$ very close to 1.

\subsection{PID Value Iteration}

The PID VI algorithm~\citep{farahmand2021pid} is designed to address the slow convergence of VI. The key observation is that the VI algorithm can be interpreted as a feedback control system with the Bellman residual as the error signal. The conventional VI corresponds to a Proportional controller, perhaps the simplest form of controller. The PID VI algorithm uses a more general PID controller~\citep{DorfBishop2008,Ogata2010} in the feedback loop instead.

A PID controller consists of three components (terms), which together determine the update of the value function from $V_k$ to $V_{k+1}$, or from $Q_k$ to $Q_{k+1}$.
The \textit{P component} is a rescaling of the Bellman residual itself, that is, $\BRpi V_k$ or $\BRstar Q_k$. The \textit{D component} is the discrete derivative of the value updates, that is, $V_k - V_{k - 1}$ or $Q_k - Q_{k-1}$. The \textit{I component} is a running average of the Bellman residuals.
The contribution of each of these terms to the value update is determined by controller gains $\kappa_p, \kappa_I, \kappa_d \in \mathbb{R}$.

To find the I component, we maintain a running average (hence, the name integration) of Bellman residual error by $z_k \colon \XX \to \reals$ for PE and $z_k \colon \XX \times \AA \to \reals$ in the Control case,
\begin{equation}
\label{eq:PIDVI-z-update}
    z_{k + 1} = \beta z_k + \alpha \BRpi V_k \quad \mathrm{(PE)} \qquad,\qquad     z_{k + 1} = \beta z_k + \alpha \BRstar Q_k \quad \mathrm{(Control)},
\end{equation}
with $\alpha , \beta \in \reals$ and $z_1$ initialized to a vector of all zeroes. PID VI  updates the value function by
\begin{align}
\label{eq:PIDVI-V-update}
    V_{k + 1} &= V_k + \kappa_p\BRpi V_k + \kappa_I(\beta z_k + \alpha \BRpi V_k) + \kappa_d(V_k - V_{k - 1}) \quad \mathrm{(PE)},\\
     Q_{k + 1} &= Q_k + \kappa_p\BRstar Q_k + \kappa_I(\beta z_k + \alpha \BRstar Q_k) + \kappa_d(Q_k - Q_{k - 1}) \quad \mathrm{(Control)}.
\end{align}
This is a generalization of the conventional VI algorithm: VI corresponds to the choice of $(\kappa_p, \kappa_I, \kappa_d) = (1, 0, 0)$.
PID VI has the same fixed point as the conventional VI, for both PE and Control. The dynamics of the sequence $(V_k)$, however, depends on the controller gains $(\kappa_p, \kappa_I, \kappa_d)$ and $(\alpha, \beta)$ of the integrator. For some choices of the gains, the dynamics converges to the fixed point, the true value function, at an accelerated rate. We also note that the dynamics is not necessarily stable, and for some gains, it might diverge.

The choice of the gains that (maximally) accelerates convergence depends on the MDP and the policy being evaluated. One approach is to place assumptions on the structure of the MDP, and analytically derive the gains that optimize the convergence rate. \cite{farahmand2021pid} provide such a result for PE in the class of reversible Markov chains.
Assuming structure on the MDP is not desirable though, so the same work also proposes a \emph{gain adaptation} algorithm that automatically tunes the controller gains during the fixed point iteration.

The proposed gain adaptation algorithm performs gradient descent at each iteration in the direction that minimizes the squared Bellman residual. For added efficiency, it is normalized by the previous Bellman residual. Formally, we pick a meta-learning rate $\eta \in \mathbb{R}$ and for each gain $\kappa_\cdot \in \{\kappa_p, \kappa_I, \kappa_d\}$, after each iteration of PID VI, we perform
\begin{align}
    \label{eq:GA-update}
    \kappa_\cdot \gets \kappa_\cdot - \eta \frac{2}{\norm{\BRpi V_k}_2^2} \frac{\partial \frac{1}{2}\norm{\BRpi V_{k + 1}}_2^2}{\partial \kappa_\cdot} = \kappa_\cdot - \eta  \frac{1}{\norm{\BRpi V_k}_2^2} \cdot \inner{\BRpi V_{k + 1},  \frac{\partial \BRpi V_{k + 1}}{\partial \kappa_\cdot}}.
\end{align}
The Control case is described similarly by substituting $\BRpi$ with $\BRstar$. 
The PID VI algorithm \eqref{eq:PIDVI-z-update}--\eqref{eq:PIDVI-V-update} and its gain adaptation procedure~\eqref{eq:GA-update} depend on the computation of the Bellman residual $\BRpi V$ or its gradient $\partial \BRpi V/\partial \kappa_\cdot$, both of which require accessing the transition dynamics $\PKernel$. PID VI, like VI, is a dynamic programming/planning algorithm after all. They are not directly applicable to the RL setting, where the agent has access only to samples from the environment that are obtained online.
The goal of the next few sections is to develop RL variants of these algorithms.

\section{PID TD Learning and PID Q-Learning}
\label{sec:pidtd-pidql}
We introduce the PID TD Learning as well as the PID Q-Learning algorithms.
These are stochastic approximation versions of the PID VI algorithm and use samples in the form of $(X_t, A_t, R_t, X'_t)$ with $A_t \sim \pi(\cdot | X_t)$ (for PE), $X_t' \sim \PKernel(\cdot|X_t,A_t)$ and $R_t \sim \RKernel(\cdot|X_t, A_t)$, instead of directly accessing $\PKernel$ and $\RKernel$.
In a typical RL setting, they form a sequence with $X_{t+1} = X'_t$.

To generalize the PID VI procedure to the sample-based setting, we first describe each iteration of PID VI with an operator. This viewpoint allows translation of PID VI to a sample-based algorithm through stochastic approximation. The same translation applied to the Bellman operator, which is the update rule for VI, yields the conventional TD Learning and Q-Learning. PID VI for PE updates three functions $V, V', z: \XX \rightarrow \Real$ at each iteration. Here, $V$ stores the value function, $V'$ stores the previous value function, and $z$ stores the running average of the Bellman errors.  For the PID Q-Learning, the domain of these functions would be $\XA$, that is, we have $Q, Q', z: \XA \rightarrow \Real$. 
We shall use $\fullV \colon \XX \times \FSet \to \reals$ and $\fullQ \colon \XX \times \AA \times \FSet \to \reals$ as a compact representation of these three functions. Here, $\FSet$ is set of size 3 that indexes the three functions contained in $\fullV$. Note that since we focus on finite MDPs, these functions can be represented by finite-dimensional vectors $\fullV \in \reals^{3n}$ and $\fullQ \in \reals^{3nm}$:
\begin{equation*}
    \fullV = \begin{bmatrix}
        V \\ z \\ V'
    \end{bmatrix} 
    \qquad,\qquad
    \fullQ = \begin{bmatrix}
        Q \\ z \\ Q'
    \end{bmatrix}. 
\end{equation*}
Also define $\fullVpi \triangleq [\Vpi \; 0 \; \Vpi ]^\top$ and $\fullQopt \triangleq [\Qopt \; 0 \; \Qopt]^\top$.
Define the space of all possible choices of gains $(\kappa_p, \kappa_I, \kappa_d, \alpha, \beta)$ to be $\GG \triangleq \reals^5$.
For a policy $\pi$ and the controller gains $g \in \GG$, we denote the PID VI operator on the space of $\BB(\XX \times \FSet)$ by $\Lp$ defined as
\begin{equation}
\label{eq:Lp-def}
    \Lp \fullV \defeq \fullV \mapsto \begin{bmatrix}
    V + \kappa_p\cdot \BRpi V 
    + \kappa_I(\beta z + \alpha \cdot \BRpi V) 
    + \kappa_d (V - V') \\[0.3em]
    \beta z + \alpha \cdot \BRpi V \\[0.3em]
    V
    \end{bmatrix}.
\end{equation}
The operator $\Ls$ on $\BB(\XA \times \FSet)$ is defined analogously, replacing $\BRpi$ with $\BRstar$. With these notations, the PID VI algorithm can be written as
\begin{align}
 \fullV_{k+1} \leftarrow \Lp \fullV_{k} \quad \text{(PE)} \qquad, \qquad \fullQ_{k+1} \leftarrow \Ls \fullQ_{k} \qquad \text{(Control)}.
\end{align}

Now we use stochastic approximation and this operator to derive our sample-based algorithms. At each iteration, the agent receives a sample $(X_t, A_t, R_t, X'_t)$ from the environment. Focus on PE and let $\fullV_t$ be the compact form of functions $V_t, z_t, V_t'$ at iteration $t$. To perform the stochastic approximation update on the value of $\fullV_t(X_t, \texttt{f})$ for some $\texttt f \in \FSet$, we need an unbiased estimator $\hat L_{t, \tf}$ of $(\Lp \fullV_t)(X_t, \texttt f)$, which is a scalar random variable. Let $N_t(x)$ and $N_t(x,a)$ be the number of times state $x$ and state-action $x,a$ are visited by time $t$. We consider the learning rate schedule $\mu \colon \mathbb{Z} \to \reals^+$ that maps the state count to the current learning rate. With the estimator $\hat L_{t, \tf}$, and state-count dependent learning rate $\mu(N_t(X_t))$, the update given by stochastic approximation is of the form
\begin{equation}   
\label{eq:general-TD-update}
\fullV_{t+1}(X_t, \texttt{f}) \leftarrow \fullV_{t}(X_t, \texttt{f}) + \mu(N_t(X_t)) (\hat L_{t, \tf} - \fullV_t(X_t, \texttt{f}) ).
\end{equation}

Note that all values of $\fullV_{t+1}$ that are not assigned an updated value will remain the same.

The only term in $(\Lp \fullV_t)(X_t, \texttt f)$ that requires estimation is $(\BRpi V_t)(X_t)$, which depends on the transition distributions of the MDP that is not available. We can form an unbiased estimate $\BRhat_t $ of $(\BRpi V_t)(X_t) = (\Tpi V)(X_t) - V(X_t)$ or $(\BRstar Q_t)(X_t, A_t) = (\Topt Q_t)(X_t, A_t) - Q_t(X_t, A_t)$ by
\begin{align}
\label{eq:br-estimate}
    & \BRhat_t =
	\begin{cases}  
	    	R_t + \gamma V_t(X_t')  - V_t(X_t) & \text{(PE)}, \\
    		R_t + \gamma \max_{a' \in \mathcal{A}} Q_t(X_t', a') - Q_t(X_t, A_t) & \text{(Control)}.
	\end{cases}
\end{align}
A stochastic approximation procedure can then be used to update the values $V_t(X_t)$, $z_t(X_t)$, $V'_t(X_t)$ according to \eqref{eq:general-TD-update}.
The procedure would be
\begin{align}
\label{eq:pid-td}
\nonumber
     V_{t+1}(X_t) &\gets V_t(X_t) + \mu(N_t(X_t)) \big[ \kappa_p \BRhat_t 
    + \kappa_I(\beta z_t(X_t) + \alpha \BRhat_t)  + \kappa_d (V_t(X_t) - V'_t(X_t)) \big], \\
\nonumber
     z_{t+1}(X_t)& \gets z_t(X_t) + \mu(N_t(X_t)) \big[ \beta z_t(X_t) + \alpha \BRhat_t - z_t(X_t) \big], \\
    V'_{t+1}(X_t) &\gets V'_t(X_t) + \mu(N_t(X_t)) \big[ V_t(X_t) - V'_t(X_t)\big].
\end{align}
We call this procedure the PID TD learning algorithm.
For Control, we similarly form estimates $\hat L_{t, \tf}$ of $(\Ls Q_t)(X_t, A_t, \tf)$ and obtain
\begin{align}
\label{eq:pid-QL}
\nonumber
    Q_{t+1}(X_t,A_t) &\gets Q_t(X_t,A_t) +  \mu_t \big[ \kappa_p \BRhat_t + 
    \kappa_I(\beta z_t(X_t,A_t) + \alpha \BRhat_t)  + \kappa_d (Q_t(X_t,A_t) - Q'_t(X_t,A_t)) \big], \\
\nonumber    
    %
     z_{t+1}(X_t,A_t) &\gets z_t(X_t,A_t) + \mu_t \big[ \beta z_t(X_t,A_t) + \alpha \BRhat_t - z_t(X_t,A_t) \big], \\
     Q'_{t+1}(X_t,A_t) &\gets Q'_t(X_t,A_t) + \mu_t \big[ Q_t(X_t,A_t) - Q'_t(X_t,A_t)\big].
\end{align}
where $\mu_t = \mu(N_t(X_t, A_t))$. This is the PID Q-Learning algorithm. It is worth mentioning that one can use other forms of learning rates to achieve better practical results. For example, we can choose constant or state-count independent learning rates, or use three different learning rates for the three updates in \eqref{eq:pid-td} and \eqref{eq:pid-QL}.  The formulation in this section is chosen for simplicity and the theoretical analysis.

\section{Theoretical Guarantees}
\label{sec:pid-theory}
In this section, we focus on the PE problem and present the theoretical analysis of PID TD Learning. We show that with proper choices of controller gains that make PID VI convergent, PID TD Learning is also convergent. Then, under synchronous update setting, we provide insights on the accelerated convergence of PID TD Learning compared to the conventional TD Learning.

\subsection{Convergence Guarantee}
\label{sec:pid-theory-convergence}
\citet{farahmand2021pid} show that PID VI converges under a wide range of gains for a wide range of environments both analytically and experimentally. We show that this convergence carries over to our sample-based PID TD Learning. We first need to define some notations to express our result. Note that $\Lp$ is an affine linear operator.
Define $\Ap$ to be its linear component and $\Bp$ to be the constant component, so that $\Lp \fullV = \Ap \fullV + \Bp$. In particular,
\begin{equation*}
    \Ap := \begin{bmatrix}
            (1 - \kappa_p + \kappa_d - \kappa_I\alpha)I+ \gamma(\kappa_p + \kappa_I\alpha)\mathcal{P}^\pi & \beta\kappa_I I & -\kappa_dI \\
            (-\alpha I + \gamma\alpha \mathcal{P}^\pi) & \beta I & 0 \\
            I & 0 & 0
        \end{bmatrix}.
\end{equation*}
The matrix $\Ap$ plays a critical role in the behavior of PID VI as well as PID TD Learning. \citet{farahmand2021pid} show that PID VI is convergent for PE if $\rho(\Ap) < 1$ where $\rho(M)$ for a square matrix $M$ is its spectral radius, the maximum of the magnitude of the eigenvalues. It turns out the condition on the controller gains needed for the convergence of PID TD Learning is weaker than the one for PID VI. We provide the following result.
\begin{theorem}[Convergence of PID TD]
\label{thm:conv-PID-TD}
Consider a set of controller gains $g$. Let $\{\lambda_i\}$ be the eigenvalues of $\Ap$.
If $\mathrm{Re}\{\lambda_i\} < 1$ for all $i$, under mild assumptions on learning rate schedule $\mu$ and the sequence $(X_t)$ (Assumptions~\ref{ass:lr_schedule},~\ref{ass:balanced_visit}), the functions $V_t$ in PID TD Learning \eqref{eq:pid-td} converge to the value function $\Vpi$ of the policy $\pi$, almost surely.
\end{theorem}

The proof of Theorem~\ref{thm:conv-PID-TD} uses the ordinary differential equations (ODE) method for convergence of stochastic approximation algorithms \citep{borkar2000ode,borkar2009stochastic}. The method binds the behavior of the stochastic approximation to a limiting ODE. In our case, the ODE is 
$$\dot{u}(t) = \Lp u(t) - u(t) = (\Ap - I) u(t) + \Bp.$$
 It is shown that if this ODE converges to the stationary point  $\fullVpi$, PID TD Learning will also converge. The condition for the convergence of this linear ODE is that the eigenvalues $\{\lambda'_i\}$ of $\Ap - I$ should have negative real parts. Since $\lambda_i' = \lambda_i -1$, we get the condition in Theorem~\ref{thm:conv-PID-TD}. Note that this condition is weaker than $\rho(\Ap) < 1$ for PID VI \citep{farahmand2021pid}, which is equivalent to $|\lambda_i| < 1$. In other words, PID TD Learning may be convergent even if PID VI with the same controller gains $g$ is not.

 Obtaining similar results for PID Q-Learning is technically much more  challenging. The reason for this difficulty is the fact that, just like PID VI for Control, as the agent's policy changes, the dynamics of PID Q-Learning changes. 
 Similar to \citet{farahmand2021pid}, we leave theoretical analysis of PID Q-Learning to future work and only focus on its empirical study.

\subsection{Acceleration Result}
\label{sec:pid-theory-acceleration}
 In this section, we provide theoretical insights on how PID TD Learning can show a faster convergence compared to the conventional TD Learning. Our analysis relies on 
the finite-sample analysis of stochastic approximation methods.
Since results for the asynchronous updates are limited, we provide our acceleration results for synchronous updates. Specifically, we provide our analysis for the case that at each iteration $t$, a dataset $\mathcal{D}_t = \{(x, A_{x, t}, R_{x, t}, X'_{x, t})\}_{x \in \XX}$ is given, where for each state $x \in \XX$ it contains the random action $A_{x, t} \sim \pi(\cdot|x)$, reward $R_{x, t} \sim \RKernel(x, A_{x, t})$, and $X'_{x, t} \sim \PKernelpi(\cdot | x, A_{x, t})$. Then, all values of $V, V'$, and $z$ are updated simultaneously in the same manner as \eqref{eq:pid-td}. Similarly, synchronous TD Learning
applies the conventional update on all states using the dataset. Based on the analysis by \citet{chen2020finite}, the following theorem provides bounds on the error of both algorithms for the learning rate schedule $\mu(t) = \epsilon / (t + T)$. We focus on the choices of $\epsilon, T$ that achieve the optimal asymptotic rate.

\begin{theorem}
	\label{thm:acceleration-PID-TD}
	Suppose synchronous TD Learning and synchronous PID TD Learning are run with initial value function $V_0$ and learning rate $\mu(t) = \epsilon/(t+T)$ to evaluate policy $\pi$. Let $\Vtd_t$ and $\Vpid_t$ be the value functions obtained by the algorithms at iteration $t$, and $\{\ctd_i, \cpid_i\}$ be constants only dependent on the MDP and controller gains. Assume $\PKernelpi$ is diagonalizable. If $\epsilon > 2/(1-\gamma)$ and $T \ge \ctd_1 \epsilon / (1-\gamma)$, we have
 \begin{align*}
\E\left[\norm{\Vtd_t - V^\pi}_\infty^2\right] 
        &\le \ctd_2 \norm{V_0 - V^\pi}_\infty^2\left(\frac{T}{t + T}\right)^{\epsilon(1 - \gamma)} + \frac{  \epsilon (\ctd_3 + \ctd_4 \norm{V^\pi}_\infty^2)}{ \epsilon(1 - \gamma) - 1}  \left(\frac{\epsilon}{t + T}\right).
 \end{align*}
 Moreover, assume we initialize $V' = V_0$ and $z = 0$ in PID TD Learning and $\Ap$ is diagonalizable with spectral radius $\rho < 1$. If $\epsilon > 2/(1-\rho)$ and $T \ge \cpid_1 \epsilon / (1-\rho)$, we have
 \begin{align*}
\E\left[\norm{\Vpid_t - V^\pi}_\infty^2\right] 
        &\le \cpid_2 \norm{V_0 - V^\pi}_\infty^2\left(\frac{T}{t + T}\right)^{\epsilon(1 - \rho)} + \frac{ \epsilon (\cpid_3  + \cpid_4  \norm{V^\pi}_\infty^2)}{ \epsilon(1 - \rho) - 1}  \left(\frac{\epsilon}{t + T}\right).
 \end{align*}
\end{theorem}

The assumption on diagonalizability of $\PKernelpi$ and $\Ap$ in Theorem~\ref{thm:acceleration-PID-TD} is for the sake of simplicity. In Appendix~\ref{sec:appendix-acceleratioon}, we provide a similar but more general result without this assumption. The upper bounds in Theorem~\ref{thm:acceleration-PID-TD} consist of two terms. The first term, which scales with the initial error $\lVert{V_0-V^\pi}\rVert_\infty$ can be interpreted as the \textit{optimization error}. It is the amount that $V_t$ still has to change to reach $V^\pi$. The second term can be considered as the \textit{statistical error}, which is independent of the initial error and exists even if we start from $V_0 = \Vpi$. Due to the conditions $\epsilon > 2/(1 - \gamma)$ and $\epsilon > 2/(1 - \rho)$, the statistical error is asymptotically dominant with rate $\OO(t^{-1})$ compared to $\OO(t^{-\epsilon(1-\gamma)})$ or $\OO(t^{-\epsilon(1-\gamma)})$ of the optimization error. Note that a larger $\epsilon$ accelerates the rate of optimization error, but together with larger $T$ (due to the condition on $T$) slows the convergence of the statistical error to zero. For example, it takes $T$ steps for the statistical error to become half of its initial value. For simplicity of discussion, we consider $\epsilon$ and $T$ fixed.

The difference between the two algorithms is in the rate that the optimization error goes to zero. This term for TD Learning is $\OO(t^{-\epsilon(1-\gamma)})$ and for PID TD Learning is $\OO(t^{-\epsilon(1-\rho)})$. When $\kappa_p = 1$ and $\kappa_I = \kappa_d = \alpha = 0$, we have $\rho = \gamma$,  and these two rates match. With a better choice of gains, one can have $\rho < \gamma$ \citep{farahmand2021pid} and achieve a faster rate for the optimization error. Even though this term is not asymptotically dominant, we show that its speed-up can be significant in the early stages of training, especially when the policy's behavior has low stochasticity. To show this, we first need to introduce the following definition. 

 \begin{definition}
     We say policy $\pi$ in MDP $(\XX, \AA, \PKernel, \RKernel, \gamma)$ is $d$-deterministic for some $d \in [0,1]$ if for all $x \in \XX$, we have $\Var[\RKernel^\pi(x)] \le (1-d)/4$ and $\max_{x'} \PKernel^\pi(x'|x) \ge d$.
 \end{definition}
 Due to our assumption that rewards are bounded within [0,1], any policy in any MDP is $0$-deterministic. The value of $d$ depends on the stochasticity of both the MDP and the policy, with a larger value corresponding to a more deterministic behavior. In the case where the policy and the MDP are both deterministic, the policy becomes $1$-deterministic. The following result shows how the initial optimization error compares to the statistical error based on this measure.

\begin{proposition}
\label{prop:error-terms-ratio}
     Assume the same conditions as in Theorem~\ref{thm:acceleration-PID-TD}. Suppose policy $\pi$ is $d$-deterministic in the environment. Let $\Eopttd(t)$ and $\Estattd(t)$ be the first and the second terms in the bound for error of TD Learning at iteration $t$, respectively. Define $\Eoptpid(t)$ and $\Estatpid(t)$ similarly. Define $c = \max((\kappa_p + \kappa_I\alpha)^2, \alpha^2)$. We have
     \begin{align*}
         \frac{\Eopttd(0)}{\Estattd(0)} \ge 
        \frac{\norm{V_0 - V^\pi}_\infty^2 (5\gamma^2n  (1-d)  + 2)  }{e n  (1-d)\left(1 + 40\gamma^2 \norm{V^\pi}_\infty^2\right)}
        \quad,\quad
        \frac{\Eoptpid(0)}{\Estatpid(0)} \ge 
        \frac{\norm{V_0 - V^\pi}_\infty^2 (15c\gamma^2n  (1-d)  + 2)  }{3e cn  (1-d)\left(1 + 40\gamma^2 \norm{V^\pi}_\infty^2\right)}.
     \end{align*}
 \end{proposition}

Proposition~\ref{prop:error-terms-ratio} shows that when the initial error $\norm{V_0 - \Vpi}_\infty$ is large or the policy behaves almost deterministically ($d$ close to $1$), the optimization error can make up the most of the error bound in Theorem~\ref{thm:acceleration-PID-TD}. In that case, the acceleration achieved by PID TD Learning in this term becomes significant in the early stages. It should be noted that our arguments in this section are based on the upper bounds on the errors of the algorithms as opposed to the errors themselves. This is a common limitation for theoretical comparisons of algorithms. In Section~\ref{sec:exps}, we further evaluate the convergence of the algorithms empirically.

\section{Gain Adaptation}
\label{sec:gain-adaptation}

The proper choice of controller gains is critical to both convergence and acceleration of our proposed algorithms. While it is possible to treat the gains as hyperparameters and tune them like any other hyperparameter, we address this by designing an automatic gain adaptation algorithm that tunes them on the fly during the runtime of the algorithm.

The design of the gain adaptation algorithm for PID TD Learning and PID Q-Learning is based on the same idea as gain adaptation in PID VI. Translating the update rule \eqref{eq:GA-update} to the sample-based settings faces two main challenges. First, the derivative $\partial \BRpi V_{k+1} / \partial \kappa_{\cdot}$ and normalization factor $\norm{\BRpi V_k}_2^2$ are not readily available without access to the transition dynamics $\PKernel$. Second, computing the inner product in \eqref{eq:GA-update} requires iterating over all states $x$, 
\begin{align*}
    \inner{\BRpi V_{k + 1},  \frac{\partial \BRpi V_{k + 1}}{\partial \kappa_\cdot}} = \sum_x (\BRpi V_{k+1})(x) \cdot  \frac{\partial (\BRpi V_{k+1})(x)}{ \partial \kappa_{\cdot} },
\end{align*}    
requiring the values of $\partial (\BRpi V_{k+1})(x) / \partial \kappa_{\cdot}$ and $(\BRpi V_{k+1})(x)$ for every $x$. We will see that using a sample $(X_t, A_t, R_t, X_t')$, these values can be estimated for $x = X_t$ but not for other states. A replay buffer could give us access to samples at more states or function approximation could directly provide estimates at all states. However, as these techniques suffer from memory and stability issues, a better solution is needed for this challenge.

To avoid the difficulty of the inner product term, we modify the update rule of the gains at iteration $t$ to minimize $(\BRpi V_{t + 1})(X_t)^2$ instead of $\norm{\BRpi V_{t + 1}}_2^2$. This modification is similar to performing stochastic gradient descent instead of gradient descent. Instead of defining the loss over the whole state space, we consider the loss on a single sampled state. Consequently, the new term only depends on the values at $X_t$. We get the following update:
\begin{align}
\label{eq:GA-update-TD1}
\kappa_\cdot &\gets \kappa_\cdot - \eta \frac{2}{\norm{\BRpi V_t}_2^2} \cdot\frac{\partial \frac{1}{2}(\BRpi V_{t + 1})(X_t)^2}{\partial \kappa_\cdot} \\
&= \kappa_\cdot - \eta  \frac{1}{\norm{\BRpi V_t}_2^2} \cdot (\BRpi V_{t + 1})(X_t) \cdot  \frac{\partial (\BRpi V_{t + 1})(X_t)}{\partial \kappa_\cdot}.
\end{align}

The term $(\BRpi V_{t + 1})(X_t)$ above can be estimated in a similar manner to \eqref{eq:br-estimate}. Estimating the derivative $\partial (\BRpi V_{t+1}) (X_t) / \partial \kappa_{\cdot}$ as well as $(\BRpi V_{t + 1})(X_t)$ in an unbiased way is another challenging problem. It is known that forming an unbiased estimate for both of these quantities with only one sample at state $X_t$ leads to double-sampling issues \citep{BAIRD199530}. As in prior work \citep{kearney2018tidbd}, we use the semi-gradient trick for this problem. Specifically, we treat the $\Tpi V$ term in $\BRpi V$ as constant and ignore its derivative. This yields the estimate
\begin{align*}
\frac{\partial (\BRpi V_{t + 1})(X_t)}{\partial \kappa_\cdot} = \frac{\partial}{\partial \kappa_\cdot} \left[ r^\pi(X_t) + \gamma \sum_{x'} \mathcal{P}^\pi(x' | X_t) V_{t+1}(x') - V_{t+1}(X_t) \right] \approx -\frac{\partial V_{t+1}(X_t)}{\partial \kappa_\cdot}.
\end{align*}
When calculating $\frac{\partial V_{t+1}(X_t)}{\partial \kappa_\cdot}$, we further ignore the effect of gains on $V_t$, setting $\frac{\partial V_{t}}{\partial \kappa_\cdot} \approx 0$, and also drop the learning rate $\mu(N_t(X_t))$ to absorb it into $\eta$. These derivatives can be calculated based on \eqref{eq:pid-td} and are given in Appendix~\ref{sec:appendix-ga}. Finally, the normalization term is estimated by keeping an exponential moving average with smoothing factor $\lambda$ of the square of estimates of the Bellman Residual in the past iterations. The detailed version of PID TD Learning and PID Q-Learning with gain adaptaion is shown in Algorithms~\ref{alg:PID-ALG} and \ref{alg:PID-Q-learning-ALG} in Appendix~\ref{sec:appendix-ga}.

\section{Empirical Results}
\label{sec:exps}

\begin{figure}[h]
    \centering
    \begin{subfigure}{0.48\linewidth}
        \includegraphics[width=\linewidth]{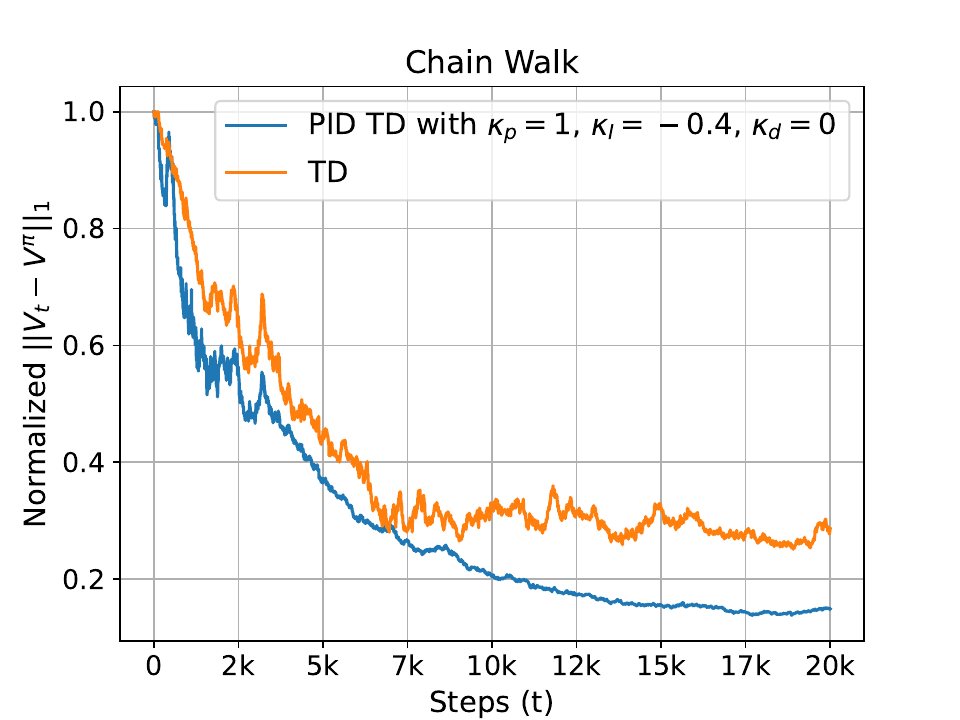}
    \end{subfigure}
    \hfill
    \begin{subfigure}{0.48\linewidth}
        \includegraphics[width=\linewidth]{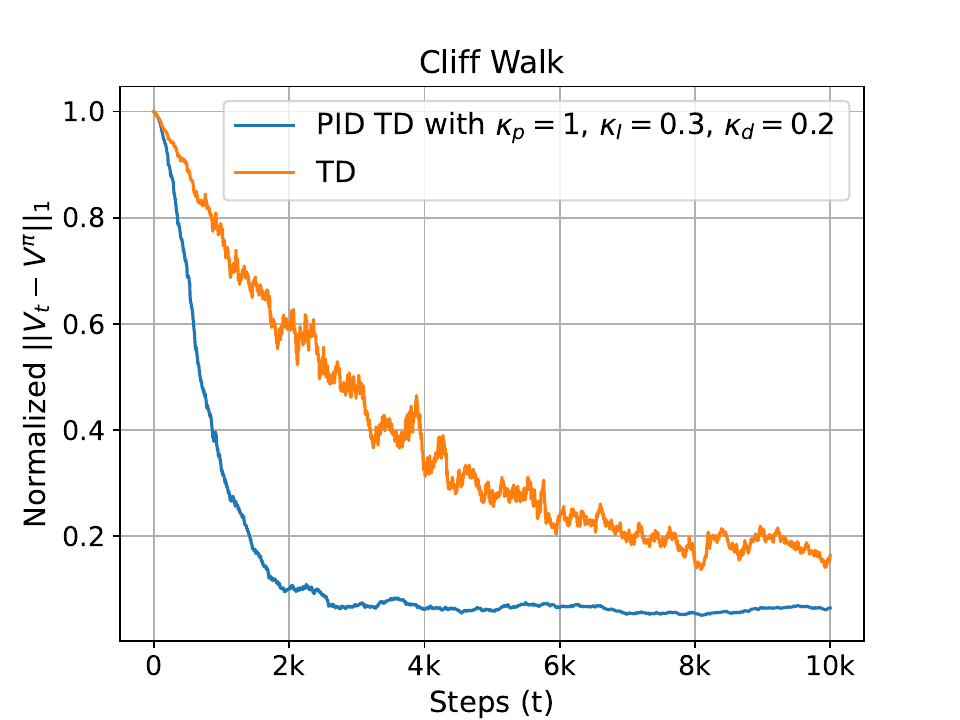}
    \end{subfigure}
    \caption{Comparison of PID TD Learning with Conventional TD Learning in Chain Walk (left) and Cliff Walk (right) with $\gamma = 0.99$. Each curve is averaged over 80 runs. Shaded areas show the standard error.}
    \label{fig:PE-fixed-gains}
\end{figure}

We empirically compare PID TD Learning and PID Q-Learning with their conventional counterparts. We conduct experiments in the 50-state Chain Walk environment with 2 actions \citep{farahmand2021pid}, the Cliff Walk environment with $6 \times 6$ states and 4 actions \citep{rakhsha2022operator}, and randomly generated Garnet MDPs with 50 states and 3 actions \citep{bhatnagar2009natural}. Detailed descriptions of these environments and the policies evaluated can be found in Appendix~\ref{sec:appendix-experiment-description}. For each sample $(X_t, A_t, R_t, X_t')$, we choose $X_t$ uniformly at random and $A_t$ is chosen according to $\pi$ (for PE) or at random (for Control). We measure the error of value functions $V_t$ and $Q_t$ for PE and Control problems by their normalized error defined as $\norm{V_t - V^\pi}_1/\norm{V^\pi}_1$ and $\norm{Q_t - Q^*}_F/\norm{Q^*}_F$, respectively, where $\norm{Q}_F \triangleq (\sum_{x,a} Q(x,a)^2)^{\frac{1}{2}}$.

For all learning rates, we use state-count dependent schedules of the form $\mu(N_t(X_t)) = \min(\epsilon, N_t(X_t)/M)$ for some choice of $\eps$ and $M$ for all algorithms (including PID Q-Learning and Q-Learning). To achieve the best results for all algorithms, we use separate learning rates for $V, V', z$ components of PID TD Learning and $Q, Q', z$ components of PID Q-Learning. The hyperparameters $\epsilon, M$ of all learning rate schedules are tuned by gridsearch over a range of values. The details of hyperparameter tuning are provided in Appendix~\ref{sec:appendix-learning-rates}.

In Figure~\ref{fig:PE-fixed-gains}, we compare PID TD Learning with TD Learning when the gains are fixed and $\gamma=0.99$. In this case the acceleration depends on the choice of gains and the environment. We observe that we can achieve a drastic acceleration in Cliff Walk, and a minor acceleration in Chain Walk. 
We further investigate the speed-up achieved by PID TD Learning in Cliff Walk. In Figure~\ref{fig:cliffwalk-adapt-gains}, we observe that with Gain Adaptation and $\gamma = 0.999$, we achieve a significant acceleration without the need to tune the controller gains. Figure~\ref{fig:cliffwalk-adapt-gains} also shows how Gain Adaptation has modified the gains from their initial values.

To evaluate the acceleration in the Control problem, we compare PID Q-Learning with Gain Adaptation with Q-Learning. Figure~\ref{fig:chainwalk-adapt-gains} shows this comparison in Chain Walk with $\gamma = 0.999$, where PID Q-Learning shows acceleration. Finally, to draw a more conclusive comparison, we compare our algorithms with the conventional ones on $80$ randomly generated Garnet MDPs with $\gamma = 0.99$ in Figure~\ref{fig:garnet-plot}. We see that our algorithms outperform TD Learning and Q-Learning in both PE and Control problems.

\begin{figure}[h]
    \centering
    \begin{subfigure}{0.48\linewidth}
        \includegraphics[width=\linewidth]{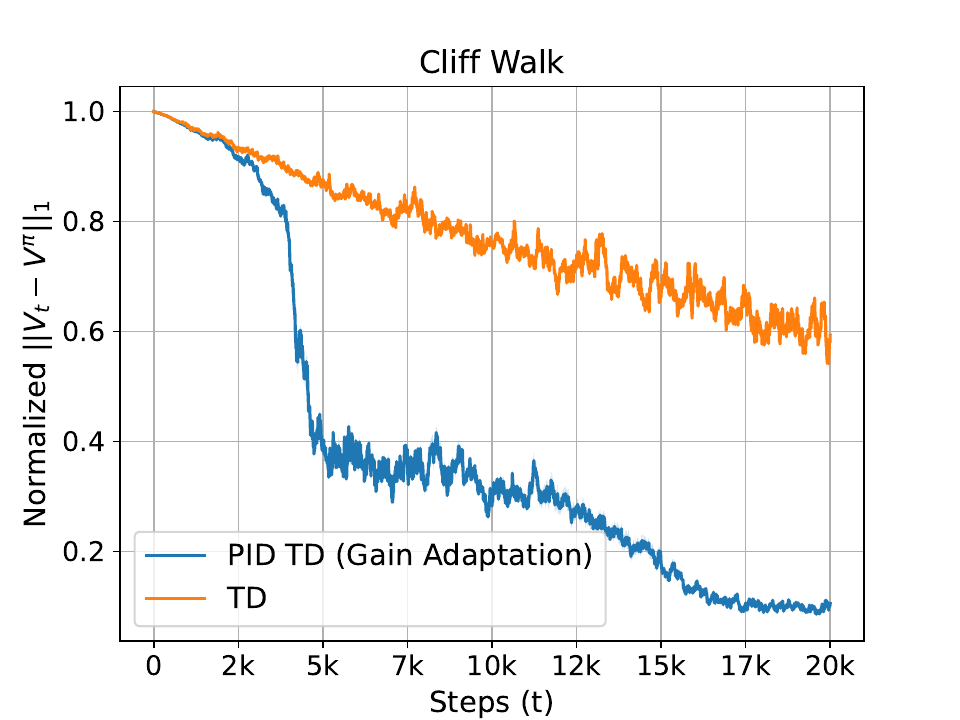}
    \end{subfigure}
    \hfill
    \begin{subfigure}{0.48\linewidth}
        \includegraphics[width=\linewidth]{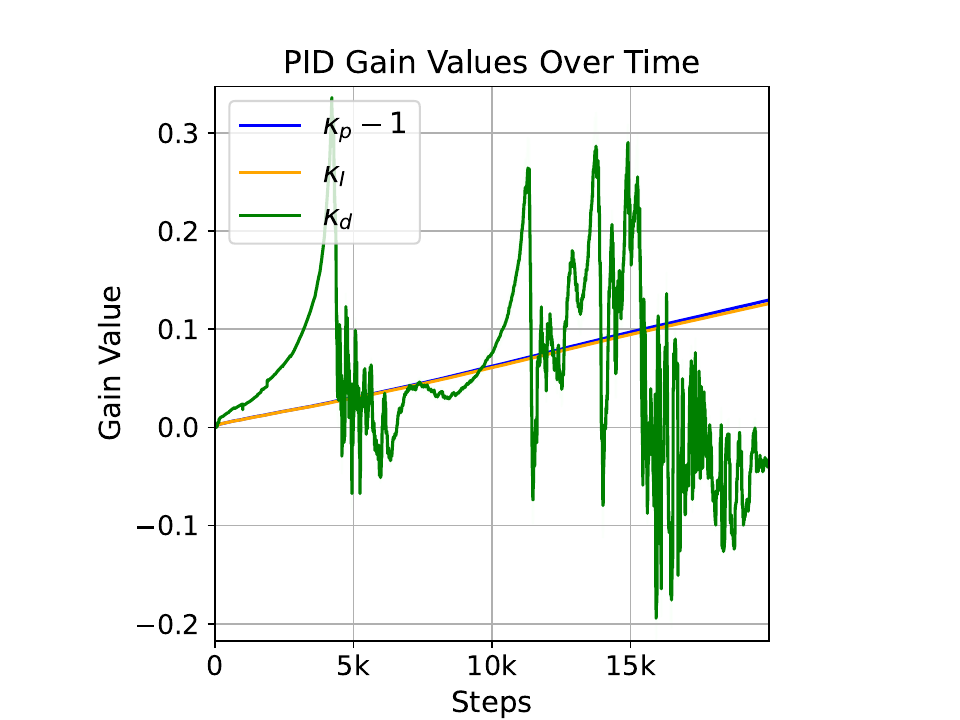}
    \end{subfigure}
    \caption{PID TD Learning with Gain Adaptation in Cliff Walk with $\gamma = 0.999$. \textit{(Left)} Comparison of value errors of PID TD Learning with TD Learning. Each curve is averaged over 80 runs. Shaded area shows standard error. \textit{(Right)} The change of gains done by Gain Adaptation through training.}
    \label{fig:cliffwalk-adapt-gains}
\end{figure}

\begin{figure}[h]
    \centering
    \begin{subfigure}{0.48\linewidth}
        \includegraphics[width=\linewidth]{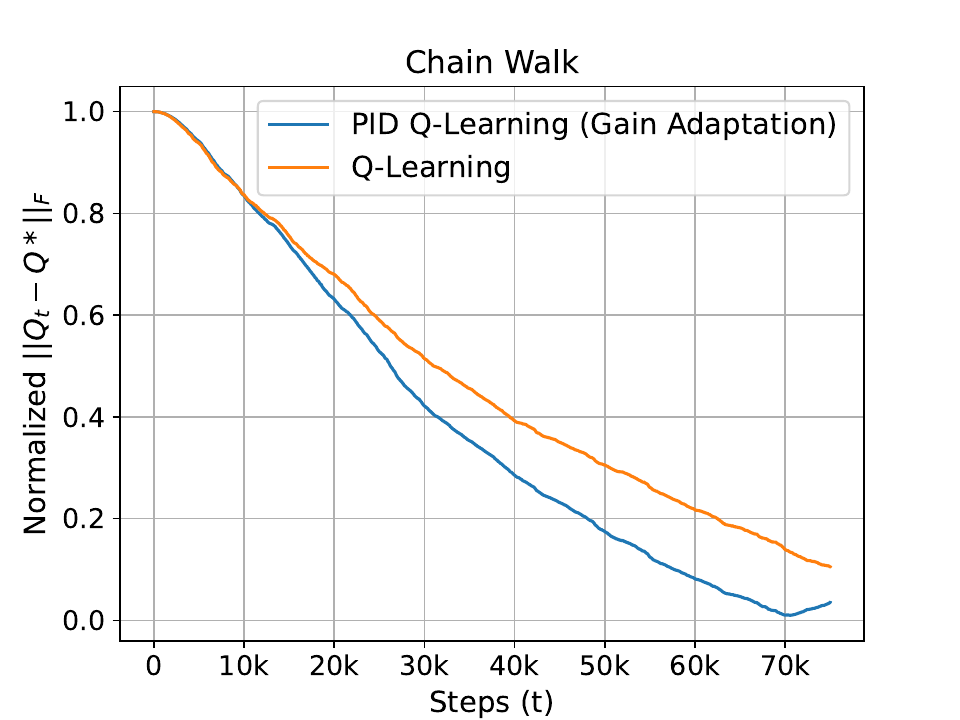}
    \end{subfigure}
    \hfill
    \begin{subfigure}{0.48\linewidth}
        \includegraphics[width=\linewidth]{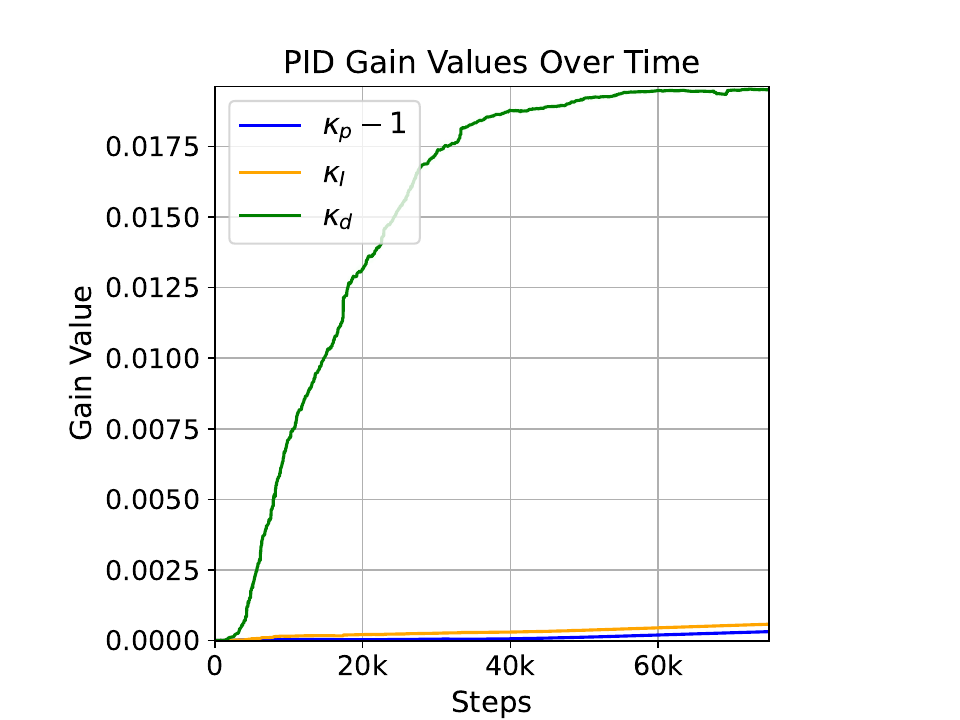}
    \end{subfigure}
    \caption{PID Q-Learning with Gain Adaptation in Chain Walk with $\gamma = 0.999$. \textit{(Left)} Comparison of value errors of PID Q-Learning with Q-Learning. Each curve is averaged over 80 runs. Shaded area shows standard error. \textit{(Right)} The change of gains done by Gain Adaptation through training.}
    \label{fig:chainwalk-adapt-gains}
\end{figure}

\begin{figure}
    \centering
    \begin{subfigure}{0.48\linewidth}
        \includegraphics[width=\linewidth]{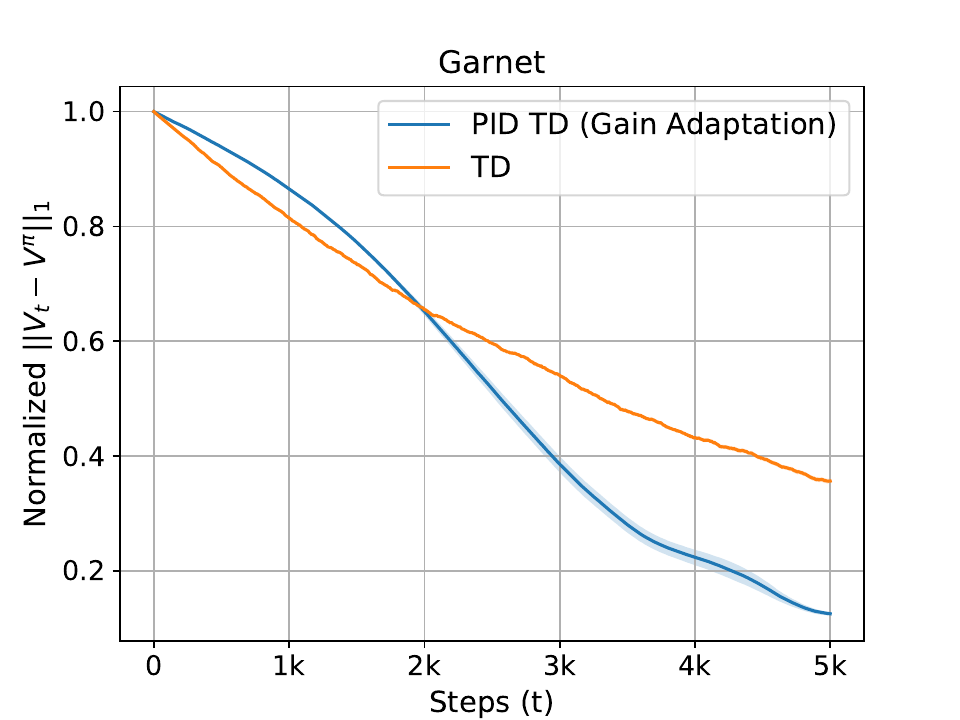}
    \end{subfigure}
    \hfill
    \begin{subfigure}{0.48\linewidth}
        \includegraphics[width=\linewidth]{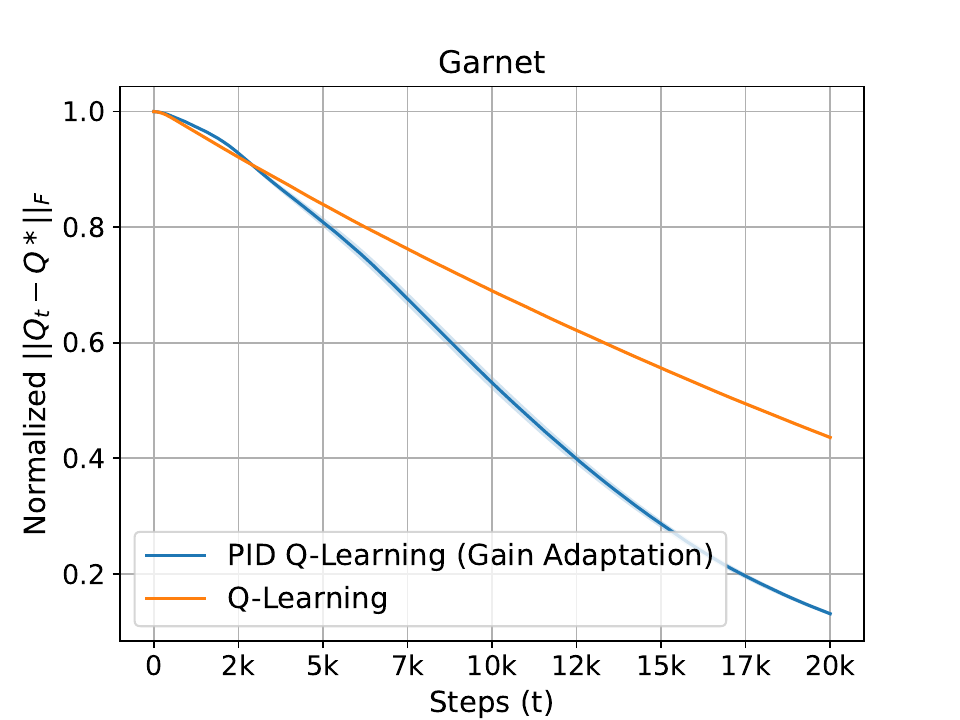}
    \end{subfigure}
    \caption{Comparison of PID Accelerated algorithms with the conventional ones for PE (Left) and Control (Right) problems in randomly generated Garnet environments with $\gamma = 0.99$. Each curve is an average of 80 MDPs, run for 80 times each. Shaded area shows standard error.}
    \label{fig:garnet-plot}
\end{figure}


\section{Related Work}
\label{sec:relatedwork}

There is a growing literature of applying acceleration techniques to RL. Similar to PID VI, which as we showed leads to PID TD Learning and PID Q-Learning, many accelerated dynamic programming methods have closely related RL algorithms. The idea of momentum has been used for faster dynamic programming algorithms by \citet{vieillard2020momentum,vineet2023accelerated} and in RL setting has led to Speedy Q-Learning \citep{ghavamzadeh2011speedy} and Momentum Q-Learning \citep{bowen2021finite}. Zap Q-Learning \citep{devraj2017zap} is another accelerated variant of Q-Learning based on second-order optimization methods. Anderson acceleration \citep{anderson1965iterative} has been used for Anderson VI \citep{geist2018anderson} and Anchoring acceleration \citep{halpern1967fixed} is used in Anchord VI \citep{lee2023accelerating}. Matrix splitting is used to derive Operator Splitting VI (OSVI) \citep{rakhsha2022operator} and Deflated Dynamics VI (DDVI) \citep{lee2024deflated}, which are both extended to the RL setting through stochastic approximation. Recently, \citet{rakhsha2024maximum} has introduced the Model Correcting VI (MoCoVI) and Model Correcting Dyna (MoCoDyna) algorithms that achieve acceleration through model correction.

Gain adaptation in general has a long history in RL and closely-related literature. \cite{kesten1958accelerated} used an adaptive mechanism in the context of stochastic approximation in the 1950s. They describe a method for choosing the learning rate of SA that is very similar to the P component of the gain adaptation procedure we naturally derive. However, the algorithm is ad hoc in nature, and is not compatible with function approximation in any natural way.
First order methods of adapting hyperparameters have been proposed, including IDBD \citep{sutton1992adapting}, the recent RL focused variant TIDBD \citep{kearney2018tidbd}, and SMD \citep{schraudolph1999local} which all tune learning rates by finding the gradient with respect to the history of errors. We refer to \cite{sutton2022history} for a more in-depth history and overview of such techniques. These approaches are limited to controlling only the learning rate of the procedure and thus only attacking the error from sampling, not the bootstrapping error.

\section{Conclusion}
\label{sec:conclusion}

We showed how recent advances in accelerated planning and dynamic programming, specifically the PID Value Iteration algorithm, can be used to design algorithms for the RL setting. 
The proposed PID TD Learning and PID Q-Learning algorithms are accompanied by a gain adaptation mechanism, which tunes their hyperparameters on the fly.
We provided theoretical analysis as well as empirical studies of these algorithms.

One limitation of the current work is that the proposed algorithms are only developed for finite MDPs where the value function, and all relevant quantities, can be represented exactly. For large MDPs, for example with continuous state spaces, we need to use function approximation. Developing PID TD Learning and PID Q-Learning with function approximation is therefore one important future direction.
Another limitation of this work is that the gain adaptation procedure, even though empirically reliable, does not come with a convergence guarantee. Moreover, small changes in its hyperparameters, such as its meta-learning rate $\eta$, can cause large changes in the trajectory the value function takes during training. Another interesting research direction is then to develop a gain adaptation procedure that is less sensitive to the choice of hyperparameters and has a convergence guarantee.
Finally, this work shows that the dynamics of RL can be significantly influenced by the PID controller, one of the simplest controllers in the arsenal of control engineering. Developing Planning and RL algorithms based on more sophisticated controllers is another promising research direction.

%

\section*{Acknowledgements and Disclosure of Funding}

We would like to thank the other members of the Adaptive Agents (Adage) Lab who provided feedback on a draft of this paper, and the anonymous reviewers whose comments helped us improve the
clarity of the paper.
AMF acknowledges the funding from the Canada CIFAR AI Chairs program, as well as the support of the Natural Sciences and Engineering Research Council of Canada (NSERC) through the Discovery Grant program (2021-03701). Resources used in preparing this research were provided, in part, by the Province of Ontario, the Government of Canada through CIFAR, and companies sponsoring the Vector Institute.

\bibliographystyle{plainnat}
\bibliography{MyBib}

\begin{thebibliography}{47}
\providecommand{\natexlab}[1]{#1}
\providecommand{\url}[1]{\texttt{#1}}
\expandafter\ifx\csname urlstyle\endcsname\relax
  \providecommand{\doi}[1]{doi: #1}\else
  \providecommand{\doi}{doi: \begingroup \urlstyle{rm}\Url}\fi

\bibitem[Anderson(1965)]{anderson1965iterative}
Donald~G. Anderson.
\newblock Iterative procedures for nonlinear integral equations.
\newblock \emph{J. ACM}, 12\penalty0 (4):\penalty0 547–560, oct 1965.
\newblock ISSN 0004-5411.
\newblock \doi{10.1145/321296.321305}.

\bibitem[Baird(1995)]{BAIRD199530}
Leemon Baird.
\newblock Residual algorithms: Reinforcement learning with function
  approximation.
\newblock In \emph{Machine Learning Proceedings 1995}, pages 30--37. Morgan
  Kaufmann, 1995.
\newblock ISBN 978-1-55860-377-6.

\bibitem[Beck(2017)]{beck2017first}
Amir Beck.
\newblock \emph{First-Order Methods in Optimization}.
\newblock MOS-SIAM Series on Optimization. Society for Industrial and Applied
  Mathematics, 2017.
\newblock ISBN 9781611974997.

\bibitem[Bertsekas and Tsitsiklis(1996)]{bertsekas1996neuro}
Dimitri~P. Bertsekas and John~N. Tsitsiklis.
\newblock \emph{Neuro-Dynamic Programming}.
\newblock Athena Scientific, 1996.
\newblock ISBN 9781886529106.

\bibitem[Bhatnagar et~al.(2009)Bhatnagar, Sutton, Ghavamzadeh, and
  Lee]{bhatnagar2009natural}
Shalabh Bhatnagar, Richard~S Sutton, Mohammad Ghavamzadeh, and Mark Lee.
\newblock Natural actor-critic algorithms.
\newblock \emph{Automatica}, 45\penalty0 (11):\penalty0 2471--2482, 2009.

\bibitem[Blitzstein and Hwang(2014)]{blitzstein2014introduction}
Joseph~K. Blitzstein and Jessica Hwang.
\newblock \emph{Introduction to Probability}.
\newblock Chapman \& Hall/CRC Texts in Statistical Science. CRC Press/Taylor \&
  Francis Group, 2014.
\newblock ISBN 9781466575578.

\bibitem[Borkar(2009)]{borkar2009stochastic}
Vivek~S. Borkar.
\newblock \emph{Stochastic Approximation: A Dynamical Systems Viewpoint}.
\newblock Texts and Readings in Mathematics. Hindustan Book Agency, 2009.
\newblock ISBN 9789386279385.

\bibitem[Borkar and Meyn(2000)]{borkar2000ode}
Vivek~S. Borkar and Sean~P. Meyn.
\newblock The {ODE} method for convergence of stochastic approximation and
  reinforcement learning.
\newblock \emph{SIAM Journal on Control and Optimization}, 38\penalty0
  (2):\penalty0 447--469, 2000.

\bibitem[Bowen et~al.(2021)Bowen, Huaqing, Lin, Yingbin, and
  Wei]{bowen2021finite}
Weng Bowen, Xiong Huaqing, Zhao Lin, Liang Yingbin, and Zhang Wei.
\newblock Finite-time theory for momentum q-learning.
\newblock In \emph{Proceedings of the Thirty-Seventh Conference on Uncertainty
  in Artificial Intelligence}, volume 161 of \emph{Proceedings of Machine
  Learning Research}. PMLR, 2021.

\bibitem[Chen and Jiang(2019)]{ChenJiang2019}
Jinglin Chen and Nan Jiang.
\newblock Information-theoretic considerations in batch reinforcement learning.
\newblock In \emph{Proceedings of the 36th International Conference on Machine
  Learning (ICML)}, 2019.

\bibitem[Chen et~al.(2020)Chen, Maguluri, Shakkottai, and
  Shanmugam]{chen2020finite}
Zaiwei Chen, Siva~Theja Maguluri, Sanjay Shakkottai, and Karthikeyan Shanmugam.
\newblock Finite-sample analysis of stochastic approximation using smooth
  convex envelopes.
\newblock \emph{arXiv preprint arXiv:2002.00874}, 2020.

\bibitem[Devraj and Meyn(2017)]{devraj2017zap}
Adithya~M. Devraj and Sean Meyn.
\newblock Zap q-learning.
\newblock \emph{Advances in Neural Information Processing Systems}, 30, 2017.

\bibitem[Dorf and Bishop(2008)]{DorfBishop2008}
Richard~C. Dorf and Robert~H. Bishop.
\newblock \emph{Modern Control Systems}.
\newblock Prentice Hall, 2008.
\newblock ISBN 9780132270281.

\bibitem[Ernst et~al.(2005)Ernst, Geurts, and Wehenkel]{Ernst05}
Damien Ernst, Pierre Geurts, and Louis Wehenkel.
\newblock Tree-based batch mode reinforcement learning.
\newblock \emph{Journal of Machine Learning Research (JMLR)}, 6:\penalty0
  503--556, 2005.

\bibitem[Even-Dar and Mansour(2003)]{EvenDarMansour2003}
Eyal Even-Dar and Yishay Mansour.
\newblock Learning rates for {Q}-learning.
\newblock \emph{Journal of Machine Learning Research (JMLR)}, 5:\penalty0
  1--25, 2003.

\bibitem[Fan et~al.(2019)Fan, Wang, Xie, and Yang]{FanWangXieYang2019}
Jianqing Fan, Zhaoran Wang, Yuchen Xie, and Zhuoran Yang.
\newblock A theoretical analysis of deep q-learning.
\newblock \emph{arXiv:1901.00137v3}, 2019.

\bibitem[Farahmand and Ghavamzadeh(2021)]{farahmand2021pid}
Amir-massoud Farahmand and Mohammad Ghavamzadeh.
\newblock {PID} accelerated value iteration algorithm.
\newblock In \emph{International Conference on Machine Learning}. PMLR, 2021.

\bibitem[Farahmand et~al.(2010)Farahmand, Szepesv{\'a}ri, and
  Munos]{FarahmandMunosSzepesvari10}
Amir-massoud Farahmand, Csaba Szepesv{\'a}ri, and R{\'e}mi Munos.
\newblock Error propagation for approximate policy and value iteration.
\newblock \emph{Advances in Neural Information Processing Systems}, 23, 2010.

\bibitem[Geist and Scherrer(2018)]{geist2018anderson}
M.~Geist and B.~Scherrer.
\newblock Anderson acceleration for reinforcement learning.
\newblock \emph{European Workshop on Reinforcement Learning}, 2018.

\bibitem[Ghavamzadeh et~al.(2011)Ghavamzadeh, Kappen, Azar, and
  Munos]{ghavamzadeh2011speedy}
Mohammad Ghavamzadeh, Hilbert Kappen, Mohammad Azar, and R\'{e}mi Munos.
\newblock Speedy q-learning.
\newblock In \emph{Advances in Neural Information Processing Systems},
  volume~24. Curran Associates, Inc., 2011.

\bibitem[Gordon(1995)]{Gordon1995}
Geoffrey Gordon.
\newblock Stable function approximation in dynamic programming.
\newblock In \emph{International Conference on Machine Learning (ICML)}, 1995.

\bibitem[Goyal and Grand-Cl\'{e}ment(2023)]{vineet2023accelerated}
Vineet Goyal and Julien Grand-Cl\'{e}ment.
\newblock A first-order approach to accelerated value iteration.
\newblock \emph{Operations Research}, 71\penalty0 (2):\penalty0 517--535, 2023.

\bibitem[Halpern(1967)]{halpern1967fixed}
Benjamin Halpern.
\newblock Fixed points of nonexpanding maps.
\newblock \emph{Bulletin of the American Mathematical Society}, 73\penalty0
  (6):\penalty0 957--961, 1967.

\bibitem[Householder(1958)]{householder1958approximate}
Alston~S. Householder.
\newblock The approximate solution of matrix problems.
\newblock \emph{Journal of the ACM (JACM)}, 5\penalty0 (3):\penalty0 205--243,
  1958.

\bibitem[Kearney et~al.(2018)Kearney, Veeriah, Travnik, Sutton, and
  Pilarski]{kearney2018tidbd}
Alex Kearney, Vivek Veeriah, Jaden~B. Travnik, Richard~S. Sutton, and
  Patrick~M. Pilarski.
\newblock Tidbd: Adapting temporal-difference step-sizes through stochastic
  meta-descent.
\newblock \emph{arXiv preprint arXiv:1804.03334}, 2018.

\bibitem[Kesten(1958)]{kesten1958accelerated}
Harry Kesten.
\newblock Accelerated stochastic approximation.
\newblock \emph{The Annals of Mathematical Statistics}, pages 41--59, 1958.

\bibitem[Lee and Ryu(2023)]{lee2023accelerating}
Jongmin Lee and Ernest~K. Ryu.
\newblock Accelerating value iteration with anchoring.
\newblock \emph{Neural Information Processing Systems}, 2023.

\bibitem[Lee et~al.(2024)Lee, Rakhsha, Ryu, and Farahmand]{lee2024deflated}
Jongmin Lee, Amin Rakhsha, Ernest~K Ryu, and Amir-massoud Farahmand.
\newblock Deflated dynamics value iteration.
\newblock \emph{arXiv preprint arXiv:2407.10454}, 2024.

\bibitem[Mnih et~al.(2015)Mnih, Kavukcuoglu, Silver, Rusu, Veness, Bellemare,
  Graves, Riedmiller, Fidjeland, Ostrovski, et~al.]{mnih2015human}
Volodymyr Mnih, Koray Kavukcuoglu, David Silver, Andrei~A Rusu, Joel Veness,
  Marc~G Bellemare, Alex Graves, Martin Riedmiller, Andreas~K Fidjeland, Georg
  Ostrovski, et~al.
\newblock Human-level control through deep reinforcement learning.
\newblock \emph{Nature}, 518\penalty0 (7540):\penalty0 529--533, 2015.

\bibitem[Munos and Szepesv\'ari(2008)]{Munos08JMLR}
R\'{e}mi Munos and {\relax Cs}aba Szepesv\'ari.
\newblock Finite-time bounds for fitted value iteration.
\newblock \emph{Journal of Machine Learning Research (JMLR)}, 9:\penalty0
  815--857, 2008.

\bibitem[Ogata(2010)]{Ogata2010}
Katsuhiko Ogata.
\newblock \emph{Modern Control Engineering}.
\newblock Prentice hall Upper Saddle River, NJ, fifth edition, 2010.

\bibitem[Rakhsha et~al.(2022)Rakhsha, Wang, Ghavamzadeh, and
  Farahmand]{rakhsha2022operator}
Amin Rakhsha, Andrew Wang, Mohammad Ghavamzadeh, and Amir-massoud Farahmand.
\newblock Operator splitting value iteration.
\newblock \emph{Advances in Neural Information Processing Systems}, 35, 2022.

\bibitem[Rakhsha et~al.(2024)Rakhsha, Kemertas, Ghavamzadeh, and massoud
  Farahmand]{rakhsha2024maximum}
Amin Rakhsha, Mete Kemertas, Mohammad Ghavamzadeh, and Amir massoud Farahmand.
\newblock Maximum entropy model correction in reinforcement learning.
\newblock In \emph{The Twelfth International Conference on Learning
  Representations}, 2024.

\bibitem[Schraudolph(1999)]{schraudolph1999local}
Nicol~N. Schraudolph.
\newblock Local gain adaptation in stochastic gradient descent.
\newblock In \emph{1999 Ninth International Conference on Artificial Neural
  Networks ICANN 99.(Conf. Publ. No. 470)}, volume~2, pages 569--574. IET,
  1999.

\bibitem[Sutton(1988)]{Sutton1988}
Richard~S. Sutton.
\newblock Learning to predict by the methods of temporal differences.
\newblock \emph{Machine Learning}, 3\penalty0 (1):\penalty0 9--44, 1988.

\bibitem[Sutton(1992)]{sutton1992adapting}
Richard~S. Sutton.
\newblock Adapting bias by gradient descent: an incremental version of
  delta-bar-delta.
\newblock In \emph{Proceedings of the Tenth National Conference on Artificial
  Intelligence}, AAAI'92, page 171–176. AAAI Press, 1992.
\newblock ISBN 0262510634.

\bibitem[Sutton(2022)]{sutton2022history}
Richard~S. Sutton.
\newblock A history of meta-gradient: Gradient methods for meta-learning.
\newblock \emph{arXiv preprint arXiv:2202.09701}, 2022.

\bibitem[Sutton and Barto(2018)]{sutton2018reinforcement}
Richard~S. Sutton and Andrew~G. Barto.
\newblock \emph{Reinforcement learning: An introduction}.
\newblock MIT press, 2018.

\bibitem[Szepesv\'ari(1997)]{Szepesvari1997}
{\relax Cs}aba Szepesv\'ari.
\newblock The asymptotic convergence-rate of {Q}-learning.
\newblock In \emph{Advances in Neural Information Processing Systems}, 1997.

\bibitem[Szepesv\'ari(2010)]{SzepesvariBook10}
{\relax Cs}aba Szepesv\'ari.
\newblock \emph{Algorithms for Reinforcement Learning}.
\newblock Morgan Claypool Publishers, 2010.

\bibitem[Teschl(2012)]{teschl2012ordinary}
Gerald Teschl.
\newblock \emph{Ordinary differential equations and dynamical systems}, volume
  140.
\newblock American Mathematical Soc., 2012.

\bibitem[Tosatto et~al.(2017)Tosatto, Pirotta, D'Eramo, and
  Restelli]{TosatooPirottaDEramoRestelli2017}
Samuele Tosatto, Matteo Pirotta, Carlo D'Eramo, and Marcello Restelli.
\newblock Boosted fitted {Q}-iteration.
\newblock In \emph{Proceedings of the 34th International Conference on Machine
  Learning (ICML)}, 2017.

\bibitem[Tsitsiklis and {Van Roy}(1997)]{TsitsiklisVanRoy97}
John~N. Tsitsiklis and Benjamin {Van Roy}.
\newblock An analysis of temporal difference learning with function
  approximation.
\newblock \emph{IEEE Transactions on Automatic Control}, 42:\penalty0 674--690,
  1997.

\bibitem[Van~Hasselt et~al.(2016)Van~Hasselt, Guez, and
  Silver]{VanHasseltGuezSilver2016}
Hado Van~Hasselt, Arthur Guez, and David Silver.
\newblock Deep reinforcement learning with double {Q}-learning.
\newblock In \emph{Proceedings of the AAAI Conference on Artificial
  Intelligence}, volume~30, 2016.

\bibitem[Vieillard et~al.(2020)Vieillard, Scherrer, Pietquin, and
  Geist]{vieillard2020momentum}
Nino Vieillard, Bruno Scherrer, Olivier Pietquin, and Matthieu Geist.
\newblock Momentum in reinforcement learning.
\newblock In \emph{International Conference on Artificial Intelligence and
  Statistics}, pages 2529--2538. PMLR, 2020.

\bibitem[Wainwright(2019)]{Wainwright2019}
Martin~J. Wainwright.
\newblock Stochastic approximation with cone-contractive operators: Sharp
  $\ell_\infty$-bounds for $q$-learning.
\newblock \emph{arXiv preprint arXiv:1905.06265}, 2019.

\bibitem[Watkins(1989)]{Watkins1989}
Christopher J. C.~H. Watkins.
\newblock \emph{Learning from Delayed Rewards}.
\newblock PhD thesis, King's College, University of Cambride, 1989.

\end{thebibliography}

\newpage
\appendix



\section{Proofs for Convergence Results (Section~\ref{sec:pid-theory-convergence})}

In this section, we present the proof of Theorem~\ref{thm:conv-PID-TD}. We first present some notation, and then the assumptions on the learning rate schedule $\mu$ and sequence of visited samples $(X_t)_{t \ge 0}$. Let $r^\pi \in \mathbb{R}^n$ be the vector of the expected immediate rewards of following the policy at each state.

Now we move on the assumptions for Theorem~\ref{thm:conv-PID-TD}.

\begin{assumption}[Properly Tapering Learning Rate Schedule] 
    \label{ass:lr_schedule}The learning rate schedule $\mu \colon \mathbb{Z} \to \reals^+$ satisfies the following:
    \begin{enumerate}[label=(\roman*)]
        \item We have $0 < \mu(t) \le 1$ for any $t \ge 0$, and 
        \begin{equation*}\sum_{t=0}^\infty \mu(t) = \infty \qquad , \qquad  \sum_{t=0}^\infty \mu(t)^2 < \infty.\end{equation*}
        \item For some $T$, we have $\mu(t + 1) < \mu(t)$ for all $t \ge T$.
        \item For $z \in (0,1)$, $\sup_t \mu([zt])/\mu(t) < \infty$, where $[\cdot]$ is the integer part of a number.
        \item For $z \in (0,1)$, 
        \begin{equation*}
            \lim_{t\to \infty} \left(\sum_{i=0}^{[zt]}\mu(i)\right) \bigg/ \left(\sum_{i=0}^{t}\mu(i)\right) = 1.
        \end{equation*}
    \end{enumerate}
\end{assumption}
Examples of learning rate schedules that satisfy Assumption~\ref{ass:lr_schedule} includes $\mu(t) = \frac{1}{t+1}$.
The next assumption is on the balanced updates of states.
\begin{assumption}[Balanced Updates of States]
    \label{ass:balanced_visit}
    The sequence of visited states $(X_t)_t$ and learning rate schedule $\mu$ is such that we have

    \begin{enumerate}[label=(\roman*)]
        \item There exists deterministic $\Delta > 0$, such that for all $x \in \XX$
        \begin{align*}
            \liminf_{t \to \infty} \frac{N_t(x)}{t} \ge \Delta  \qquad \mathrm{a.s.}
        \end{align*}
        \item If $T_t(z) \triangleq \min \{t' > t \colon \sum_{i=t+1}^{t'} \mu(i) > z\}$, for any $z > 0$ and states $x_1, x_2 \in \XX$, the following limit exists
        \begin{equation*}
            \lim_{t \to \infty} \frac{ \sum_{i=N_t(x_1)}^{N_{T_t(z)}(x_1)}\mu(i) }{\sum_{i=N_t(x_2)}^{N_{T_t(z)}(x_2)}\mu(i)}.
        \end{equation*}
    \end{enumerate}
\end{assumption}
Intuitively, Assumption~\ref{ass:balanced_visit} asserts that all states are visited often enough and get balanced sum of learning rates.
Before presenting the proof for Theorem~\ref{thm:conv-PID-TD}, we first prove the following auxiliary lemma.
\begin{lemma}
    \label{lem:PID-noise-linear-boundedness} Assume policy $\pi$ in the environment is $d$-deterministic and $x \in \XX$ is arbitrary. Let $R$ and $X'$ be the random obtained reward and next state after following policy $\pi$ from $x$ in the environment. Let
    $W = R + \gamma V(X') - (T^\pi V)(x)$ for an arbitrary $V \colon \XX \to \reals$. We have
    \begin{equation*}
        \E\left[W^2\right] \leq \frac{1-d}{4} + 5\gamma^2(1-d)\norm{V}^2_\infty.
    \end{equation*}
    Moreover, for some $\fullV, \texttt{f}$, let $\hat L$ be the estimator of $(\Lp \fullV)(x, \texttt{f})$ derived in PID TD Learning's update \eqref{eq:general-TD-update} according to the sample $(x, R, X')$. Assume $\tilde{W} = \hat L - (\Lp \fullV)(x, \texttt{f})$ is its noise. We have
    \begin{equation*}
        \E\left[\tilde{W}^2\right] \leq \max((\kappa_p + \kappa_I\alpha)^2, \alpha^2)\left(\frac{1-d}{4} + 5\gamma^2(1-d)\norm{\fullV}^2_\infty\right).
    \end{equation*}
\end{lemma}
\begin{proof}
For the first part, we write
\begin{align*}
    \E[W^2]
    &= \E\left[(\;R + \gamma V(X') - r^\pi(x) - \gamma \E[V(X')]\;)^2\right]\\
    &= \E\left[(R - r^\pi(X))^2\right] + \gamma^2 \E\left[ (\; V(X') - \E[V(X')] \;)^2\right] +   2\E[(R - r^\pi(x)) (V(X') - \E[V(X')])]\\
    &= \Var[R] + \gamma^2 \Var[V(X')]\\
    &\le \frac{1-d}{4} + \gamma^2 \Var[V(X')],
\end{align*}
where the last inequality is by the definition of $d$-deterministic MDP. Now let $p^* = \max_{x'} \PKernelpi(x'|x)$ and $x^* = \argmax_{x'} \PKernelpi(x'|x)$. Due to the law of total variance \citep[Example 9.5.5]{blitzstein2014introduction}, we have
\begin{align*}
    \Var[V(X')] 
    &=  p^* \Var[V(X') | X'=x^*]  + (1 - p^*) \Var[V(X') | X'\ne x^*] \\
    &\qquad \quad + p^*(1-p^*)\bigg(\E[V(X') | X' =x*] - \E[V(X') | X' \ne x*] \bigg)^2\\
    &\le (1-p^*)\norm{V}_\infty^2 + 4p^*(1-p^*) \norm{V}_\infty^2 \\
    &\le 5(1-d) \norm{V}_\infty^2.
\end{align*}
Together, we obtain 
\begin{align*}
    \E[W^2] \le \frac{1-d}{4} + 5(1-d) \norm{V}_\infty^2.
    \end{align*}
For the second part, we consider three cases. If $\texttt f = \texttt v$, we have
\begin{equation*}
    \tilde W = (\kappa_p + \kappa_I\alpha) W.
\end{equation*}
If $\texttt f = \texttt z$, we have
\begin{equation*}
    \tilde W = \alpha W.
\end{equation*}
If $\texttt f = \texttt v'$, we have
\begin{equation*}
    \tilde W = 0.
\end{equation*}
Combining all cases, we get
    \begin{align*}
        \tilde{W}^2 \le \max((\kappa_p + \kappa_I\alpha)^2, \alpha^2) W^2,
    \end{align*}
which means
\begin{align*}
       \E\left[\tilde{W}^2\right]  \le \max((\kappa_p + \kappa_I\alpha)^2, \alpha^2) \left(\frac{1-d}{4} + 5\gamma^2(1-d)\norm{\fullV}^2_\infty\right).
    \end{align*}
\end{proof}

\subsection*{Proof of Theorem~\ref{thm:conv-PID-TD}}
\begin{proof}

    We show the claim by applying the result by \citet[Theorem 2.5]{borkar2000ode} to our algorithm. We describe how PID TD Learning \eqref{eq:pid-td} is a special case of a convergent asynchronous stochastic approximation in \citep{borkar2000ode}. PID TD Learning updates three entries of $\fullV_t$ at each iteration. Therefore, our set of indices that are updated (noted by $Y(n)$ in the original paper) is $Y_t \triangleq \{(X_t, \tv), (X_t, \tz), (X_t, \tvp)\}$. The number of times the value for $(x, \tf) \in \XX \times \{\tv, \tz, \tvp\}$ is updated (noted by $\nu(i, n)$ in the original paper) is $N_t(x)$ and the communication delays are zero in our case. With these choices, the asynchronous stochastic approximation of \citep[Equation 2.8]{borkar2000ode} becomes
    \begin{align}
        \fullV_{t+1}(X_t, \tf) \leftarrow \fullV_{t}(X_t, \tf) + \mu(N_t(X_t)) f[\fullV_t, D_{t}](X_t, \tf) \quad (\forall \tf \in \{\tv, \tz, \tvp\}),
    \end{align}
    and the other entries in $\fullV_{t+1}$ remain the same as $\fullV_t$. Here, $D_t \in \DD$ for all $t$ are independently and identically distributed (i.i.d.), and $f \colon \reals^{3n} \times \DD \to \reals^{3n}$ can be an arbitrary function. To obtain PID TD Learning in this form, we define $\DD \triangleq (\XX \times [0,1])^n$. For any $t$, we choose $D_{t}$ to be a dataset $\{(R_{x, t}, X'_{x, t})\}_{x \in \XX}$ where for each state $x$ contains the random reward $R_{x, t} \sim \RKernel^\pi(x)$ and $X'_{x, t} \sim \PKernelpi(\cdot | x)$. Then we define $f$ such that for all $x \in \XX$,
    \begin{align*}
        &f[\fullV_t, D_{t}](x, \tv) \triangleq \kappa_p b_{x,t} 
    + \kappa_I(\beta z_t(x) + \alpha b_{x,t})  + \kappa_d (V_t(x) - V'_t(x)),\\
    &f[\fullV_t, D_{t}](x, \tz) \triangleq \beta z_t(x) + \alpha b_{x,t}- z_t(x),\\
    &f[\fullV_t, D_{t}](x, \tvp) \triangleq  V_t(x) - V'_t(x).
    \end{align*}
    where 
    $
    b_{x,t} = R_{x,t} + \gamma V_t(X'_{x,t}) - V_t(x)
    $. This yields the exact same PID TD Learning updates.
    
    The function $h(\fullV) = \mathbb{E}[f(\fullV, D_1)]$ in \citep{borkar2000ode} is equal to $\Lp \fullV - \fullV$ in our setting. Note that $h$ is Lipschitz since it is an affine linear operator.
    The function $h_\infty(\fullV)$ exists and is equal to $(\Ap - I)\fullV$. Therefore, we require the origin point to be an asymptotically stable equilibrium of the ODE
     \begin{equation*}
        \dot{u}(t) = h_\infty(u(t)) = (\Ap - I)u(t),
    \end{equation*}
    which is satisfied due to the assumption on the eigenvalues of $\Ap$ and the fact that the solution of the above ODE is $\exp{[(\Ap - I)t]}u_0$ \citep{teschl2012ordinary} for any starting point $u_0$. 
    
    Furthermore, we note that  the unique globally asymptotically stable equilibrium of ODE
     \begin{equation*}
        \dot{u}(t) = h(u(t)) = \Lp u(t) - u(t) = (\Ap - I)u(t) + \Bp.
    \end{equation*}
   is $-(\Ap - I)^{-1}\Bp$ which is equal to $\fullVpi$ due to $\fullVpi = \Lp \fullVpi = \Ap \fullVpi + \Bp$.

    Finally, note that Lemma~\ref{lem:PID-noise-linear-boundedness} established the required property of the noise. The remaining assumptions of \citet{borkar2000ode} are satisfied due to Assumption~\ref{ass:lr_schedule} and \ref{ass:balanced_visit}.

\end{proof}

\section{Proofs for Acceleration Results (Section~\ref{sec:pid-theory-acceleration})}
\label{sec:appendix-acceleratioon}

Before presenting the proof of the Theorems, we introduce these definitions.

\begin{definition}
    Let $f \colon \reals^d$ be a convex, differentiable function. Then $f$ is said to be $L$-smooth with respect to (w.r.t.) norm $\norm{\cdot}$ if and only if
    \begin{align*}
        f(y) \le f(x) + \inner{\nabla f(x), y - x} + \frac{L}{2} \norm{y - x}^2 \qquad \forall x,y \in \reals^d.
    \end{align*}
\end{definition}

\begin{definition}
    For any non-singular matrix $S \in \reals^{d\times d}$, we define the vector norm $\norm{v}_{2, S} \triangleq \norm{Sv}_2$. For any matrix $A$, we let $\norm{A}_{2,S}$ be the matrix norm of $A$ induced by the vector norm $\norm{\cdot}_{2,S}$.
\end{definition}

We also need the following lemmas.

\begin{lemma}
    \label{lemma:approx-norm}
    Let $A \in \reals^{d \times d}$ and $\delta \ge 0$. If $A$ is not diagonalizable, further assume $\delta > 0$. There exists an invertible matrix $S$ such that $\norm{A}_{2, S} \le \rho(A) + \delta$ and for any $v \in \reals^d$, $\norm{v}_{2, S} \le \norm{v}_\infty$.
\end{lemma}
\begin{proof}
    The existence of $S'$ such that $\norm{A}_{2, S'} \le \rho(A) + \delta$ is a consequence of the proof of Theorem~4.4 in \citet{householder1958approximate}. Due to equivalence of norms in finite dimensions, there exists $u > 0$ such that for any $v \in \reals^d$, $\norm{v}_{2, S'} \le u\norm{v}_\infty$. Define $S = \frac{1}{u}S'$. Consequently, $\norm{v}_{2, S} \le \norm{v}_\infty$ for any $v$. Now from Theorem 2.10 in \citet{householder1958approximate}, we have $\norm{A}_{2, S'} = \norm{S'AS'^{-1}}_2$ and $\norm{A}_{2, S} = \norm{SAS^{-1}}_2$ which means
    \begin{align*}
        \norm{A}_{2, S} = \norm{SAS^{-1}}_2 = \norm{S'AS'^{-1}}_2 = \norm{A}_{2, S'} \le  \rho(A) + \delta.
    \end{align*}
\end{proof}

\begin{lemma}
    \label{lemma:smooth-nrom}
    For any invertible matrix $S$, the function $f \colon \reals^d \to \reals$ defined as $f(x) = \frac{1}{2} \norm{x}_{2,S}^2$ is $1$-smooth w.r.t.  $\norm{\cdot}_{2,S}$.
\end{lemma}
\begin{proof}
Let $g(x) = \frac{1}{2}\norm{x}_2^2$. By definition, $f(x) = g(Sx)$. We write
\begin{align*}
    f(x) + \inner{\nabla f(x), y - x} + \frac{1}{2} \norm{y - x}^2_{2, S}  &= g(Sx) + \inner{S^\top \nabla g(Sx), y - x} + \frac{1}{2} \norm{Sy - Sx}_2^2\\
    &= g(Sx) + \inner{\nabla g(Sx), Sy - Sx} + \frac{1}{2} \norm{Sy - Sx}^2_2.
\end{align*}
    By \citet[Example 5.11]{beck2017first}, $g$ is $1$-smooth, that is for any $u, v$,
    \begin{align*}
        g(u) + \inner{\nabla g(u),v - u} + \frac{1}{2} \norm{v - u}^2_2 \ge g(v).
    \end{align*}
    Together, we conclude 
    \begin{align*}
        f(x) + \inner{\nabla f(x), y - x} + \frac{1}{2} \norm{y - x}^2_{2, S}   \ge g(Sy) = f(y).
    \end{align*}
\end{proof}

\begin{lemma}
    \label{lemma:PID-noise-linear-boundedness-vector}
    Assume a dataset $\{(x, R_{x}, X'_{x})\}_{x}$ is given, where for each state $x$ contains the random reward $R_{x} \sim \RKernel^\pi(x)$ and $X'_{x} \sim \PKernelpi(\cdot | x)$, and $\pi$ is $d$-deterministic in the environment. For an arbitrary value function $V$, define $W \colon \XX \to \reals$ as
    \begin{align*}
        W(x) \triangleq R_x + \gamma V(X'_x) - r^\pi(x) - \gamma (\PKernelpi V)(x).
    \end{align*}
    Moreover, for some $\fullV$ and $x, \texttt{f}$, let $\hat L(x, \texttt f)$ be the estimator of $(\Lp \fullV)(x, \texttt{f})$ derived in PID TD Learning's update   \eqref{eq:general-TD-update} according to the sample $(x, R_x, X'_x)$. Define 
    $$\tilde{W}(x, \texttt f) = \hat L(x, \texttt f) - (\Lp \fullV)(x, \texttt{f}).$$
    We have
    \begin{gather*}
    \E\left[\norm{W}_\infty^2\right] \leq n\left(\frac{1-d}{4} + 5\gamma^2(1-d)\norm{V}^2_\infty\right),\\
        \E\left[\norm{\tilde{W}}_\infty^2\right] \leq 3n\max((\kappa_p + \kappa_I\alpha)^2, \alpha^2)\left(\frac{1-d}{4} + 5\gamma^2(1-d)\norm{\fullV}^2_\infty\right).
    \end{gather*}
\end{lemma}
\begin{proof}
    According to Lemma~\ref{lem:PID-noise-linear-boundedness}, for any $x , \texttt f$,
    \begin{gather*}
    \E\left[W(x)^2\right] \leq \frac{1-d}{4} + 5\gamma^2(1-d)\norm{V}^2_\infty,\\
        \E\left[\tilde{W}(x, \texttt f)^2\right] \leq \max((\kappa_p + \kappa_I\alpha)^2, \alpha^2)\left(\frac{1-d}{4} + 5\gamma^2(1-d)\norm{\fullV}^2_\infty\right).
    \end{gather*}
    The result follows from the fact that for any random vector $Z = [Z_1, \ldots, Z_k]^\top$, 
    \begin{align*}
        \E\left[\norm{Z}_\infty^2\right] \le \E\left[\sum_i Z_i^2\right] = \sum_i E\left[Z_i^2\right].
    \end{align*}
\end{proof}

\subsection{Proof of Theorem~\ref{thm:acceleration-PID-TD}}
We prove the more general result than Theorem~\ref{thm:acceleration-PID-TD} without any diagonalizablity assumptions. Theorem~\ref{thm:acceleration-PID-TD} is the special case of the following when $\deltatd = \deltapid = 0$.
\begin{theorem}
	\label{thm:acceleration-PID-TD-general}
	Suppose synchronous TD Learning and PID TD Learning are run with initial value function $V_0$ and learning rate $\mu(t) = \epsilon/(t+T)$ to evaluate policy $\pi$. Let $\Vtd_t, \Vpid_t$ be the value functions obtained with each algorithm at iteration $t$. Assume $\deltatd, \deltapid \ge 0$. If $\PKernelpi$ is not diagonalizable, we further assume $\deltatd > 0$, and if $\Ap$ is not diagonalizable, we assume $\deltapid > 0$. If $\epsilon > 2/(1-\gamma - \deltatd)$ and $T \ge \cpid_1 \epsilon / (1-\gamma - \deltatd)$, we have
 \begin{align*}
\E\left[\norm{\Vtd_t - V^\pi}_\infty^2\right] 
        &\le \cpid_2 \norm{V_0 - V^\pi}_\infty^2\left(\frac{T}{t + T}\right)^{\epsilon(1 - \gamma - \deltatd)} + \frac{  \epsilon (\cpid_3 + \cpid_4 \norm{V^\pi}_\infty^2)}{ \epsilon(1 - \gamma - \deltatd) - 1}  \cdot \left(\frac{\epsilon}{t + T}\right).
 \end{align*}
 Here, $\{\cpid_i\}$ are constants dependent on the MDP and $\deltatd$. Moreover, assume we initialize $V' = V_0, z \equiv 0$ in PID TD Learning and $\Ap$ has spectral radius $\rho < 1$. If $\epsilon > 2/(1-\rho - \deltapid)$ and $T \ge \cpid_1 \epsilon / (1-\rho - \deltapid)$, we have
 \begin{align*}
\E\left[\norm{\Vpid_t - V^\pi}_\infty^2\right] 
        &\le \cpid_2 \norm{V_0 - V^\pi}_\infty^2\left(\frac{T}{t + T}\right)^{\epsilon(1 - \rho - \deltapid)} + \frac{ \epsilon (\cpid_3  + \cpid_4  \norm{V^\pi}_\infty^2)}{ \epsilon(1 - \rho - \deltapid) - 1}  \cdot \left(\frac{\epsilon}{t + T}\right).
 \end{align*}
  Here, $\{\cpid_i\}$ are constants dependent on the MDP, controller gains, and $\deltapid$.
\end{theorem}

\subsubsection*{Proof of Theorem~\ref{thm:acceleration-PID-TD-general} for TD Learning}

\begin{proof}
     Since $\PKernelpi$ is a stochastic matrix, we have $\rho(\PKernelpi) = 1$. Based on Lemma~\ref{lemma:approx-norm}, let $S$ be the non-singular matrix such that $\norm{\gamma \PKernelpi}_{2, S} \le \gamma + \deltatd$. For any two $V_1$ and $V_2$ we have
    \begin{align*}
        \snorm{T^\pi V_1 - T^\pi V_2} = \snorm{\gamma \PKernelpi (V_1 - V_2)} \le \snorm{\gamma \PKernelpi} \snorm{V_1 - V_2} \le  (\gamma + \deltatd) \snorm{V_1 - V_2},
    \end{align*}
    which means $T^\pi$ is a $(\gamma + \deltatd)$-contraction w.r.t. $\norm{\cdot}_{2, S}$. Moreover, $\frac{1}{2}\snorm{x}^2$ is $1$-smooth according to Lemma~\ref{lemma:smooth-nrom}. Consequently, we use $\norm{\cdot}_{2, S}$ as the norms $\norm{\cdot}_c$ and $\norm{\cdot}_s$ in \citet{chen2020finite}.

    \newcommand{\TD}[1]{V_{#1}^{(\text{TD})}}

    Assume the policy is $d$-deterministic. Define the noise $W_t \colon \XX \to \reals$ at iteration $t$ as
    \begin{align*}
        W_t(x) \triangleq R_{x, t} + \gamma V_t(X'_{x,t}) - \rpi(x) - \gamma (\PKernelpi V_t)(x).
    \end{align*}
    By Lemma~\ref{lemma:PID-noise-linear-boundedness-vector}, we can bound the conditional noise at each iteration as
     $$\E\left[\norm{W_t}_\infty^2 \middle | V_0, W_0, \ldots , W_{t-1}, V_t \right] \leq \Ctd + \Btd\norm{V_t}^2_\infty,$$
     where
     \begin{align*}
         \Ctd \triangleq \frac{n(1-d)}{4}, \qquad
         \Btd \triangleq 5\gamma^2n  (1-d).
     \end{align*}

    There exists a constant $\ltd > 0$ by the equivalence of norms and Lemma~\ref{lemma:approx-norm} such that for all $x\in \mathbb{R}^{n}$:
    \begin{equation}
        \label{eq:equivalence-of-norms-lem-acceleration-TD}
        \ltd\norm{x}_\infty \leq \snorm{x} \leq \norm{x}_\infty.
    \end{equation}
    Based on these, define the constant
    \begin{equation*}
        \ctd_1 \triangleq \frac{8 (\Btd  + 2)}{\ltd^2}.
    \end{equation*}
    
    By Corollary 2.1.2 of \citet{chen2020finite}, choosing the norm $\norm{\cdot}_e$ to be $\norm{\cdot}_\infty$, $\mu = L = 1$, the error at iteration $t$ is bounded as
    \begin{align*}
        \E\left[\snorm{V_t - V^\pi}^2\right] &\leq \snorm{V_0 - V^\pi}^2\left(\frac{T}{t + T}\right)^{\epsilon(1 - \gamma - \deltatd)} + \frac{16e \epsilon^2\left(\Ctd + 2\Btd \snorm{V^\pi}^2\right)}{\ltd^2 (\epsilon(1 - \gamma - \deltatd) - 1)}  \cdot \left(\frac{1}{t + T}\right).
    \end{align*}

    By Equation~\eqref{eq:equivalence-of-norms-lem-acceleration-TD}, this immediately gives us a bound on the infinity norm.
    \begin{align*}
        \E\left[\norm{V_t - V^\pi}_\infty^2\right] 
        &\leq \frac{1}{\ltd^2}  \E\left[\snorm{V_t - V^\pi}^2\right]\\
        &\leq \frac{1}{\ltd^2}  \snorm{V_0 - V^\pi}^2\left(\frac{T}{t + T}\right)^{\epsilon(1 - \gamma - \deltatd)}+ \frac{16e \epsilon^2\left(\Ctd + 2\Btd \snorm{V^\pi}^2\right)}{\ltd^4 (\epsilon(1 - \gamma - \deltatd) - 1)}  \cdot \left(\frac{1}{t + T}\right)\\
        &\leq \frac{1}{\ltd^2}  \norm{V_0 - V^\pi}_\infty^2\left(\frac{T}{t + T}\right)^{\epsilon(1 - \gamma - \deltatd)}+ \frac{16e \epsilon^2\left(\Ctd + 2\Btd \norm{V^\pi}_\infty^2\right)}{\ltd^4 (\epsilon(1 - \gamma - \deltatd) - 1)}  \cdot \left(\frac{1}{t + T}\right).
    \end{align*}
    This gives the statement of theorem by defining
    \begin{gather*}
        \ctd_2 \triangleq \frac{1}{\ltd^2}, \quad
        \ctd_3 \triangleq \frac{16e\Ctd}{\ltd^4}, \quad
        \ctd_4 \triangleq \frac{32e\Btd}{\ltd^4}.
    \end{gather*}
        
\end{proof}

\subsubsection*{Proof of Theorem~\ref{thm:acceleration-PID-TD-general} for PID TD Learning}
\begin{proof}
    The proof follows the exact steps in the proof for TD Learning. Based on Lemma~\ref{lemma:approx-norm}, let $S$ be the non-singular matrix such that $\norm{\Ap}_{2, S} \le \rho + \deltapid$. For any two $\fullV_1$ and $\fullV_2$ we have
    \begin{align*}
        \snorm{\Lp \fullV_1 - \Lp \fullV_2} = \snorm{\Ap (\fullV_1 - \fullV_2)} \le \snorm{\Ap} \snorm{\fullV_1 - \fullV_2} \le  (\rho + \deltapid) \snorm{\fullV_1 - \fullV_2},
    \end{align*}
    which means $\Lp$ is a $(\rho + \deltapid)$-contraction w.r.t. $\norm{\cdot}_{2, S}$. Moreover, $\frac{1}{2}\snorm{x}^2$ is $1$-smooth according to Lemma~\ref{lemma:smooth-nrom}. Consequently, we use $\norm{\cdot}_{2, S}$ as the norms $\norm{\cdot}_c$ and $\norm{\cdot}_s$ in \citet{chen2020finite}.

    \newcommand{\TD}[1]{V_{#1}^{(\text{TD})}}

    Assume the policy is $d$-deterministic. Define the noise $\tilde W_t \colon \XX \times \{\tv, \tz, \tvp\}\to \reals$ at iteration $t$ as
    \begin{align*}
        \tilde W_t(x, \tv) &\triangleq (\kappa_p + \kappa_I \alpha)(R_{x, t} + \gamma V_t(X'_{x,t}) - \rpi(x) - \gamma (\PKernelpi V_t)(x)),\\
        \tilde W_t(x, \tz) &\triangleq \alpha(R_{x, t} + \gamma V_t(X'_{x,t}) - \rpi(x) - \gamma (\PKernelpi V_t)(x)),\\
        \tilde W_t(x, \tvp) &\triangleq 0.
    \end{align*} 
    By Lemma~\ref{lemma:PID-noise-linear-boundedness-vector}, we can bound the noise at each iteration as
     $$\E\left[\norm{\tilde{W}_t}_\infty^2 \middle | \tilde V_0, \tilde W_0, \ldots , \tilde W_{t-1}, \tilde V_t \right] \leq \Cpid + \Bpid\norm{\fullV_t}^2_\infty,$$
     where
     \begin{align*}
         \Cpid \triangleq \frac{3n}{4}\max((\kappa_p + \kappa_I\alpha)^2, \alpha^2)(1-d), \qquad
         \Bpid \triangleq 15\gamma^2n \cdot \max((\kappa_p + \kappa_I\alpha)^2, \alpha^2)(1-d).
     \end{align*}

    There exists a constant $l > 0$ by the equivalence of norms and Lemma~\ref{lemma:approx-norm} such that for all $x\in \mathbb{R}^{3n}$:
    \begin{equation}
        \label{eq:equivalence-of-norms-lem-acceleration-PID-TD}
        \lpid\norm{x}_\infty \leq \snorm{x} \leq \norm{x}_\infty.
    \end{equation}
    Based on these, define the constants
    \begin{equation*}
        \cpid_1 \triangleq \frac{8 (\Bpid  + 2)}{\lpid^2}.
    \end{equation*}
    
    By Corollary 2.1.2 of \citet{chen2020finite}, choosing the norm $\norm{\cdot}_e$ to be $\norm{\cdot}_\infty$, the error at iteration $t$ is bounded as
    \begin{align*}
        \E\left[\snorm{\fullV_t - \fullV^\pi}^2\right] &\leq \snorm{\fullV_0 - \fullV^\pi}^2\left(\frac{T}{t + T}\right)^{\epsilon(1 - \rho - \deltapid)} + \frac{16e \epsilon^2\left(\Cpid + 2\Bpid \snorm{\fullV^\pi}^2\right)}{\lpid^2 (\epsilon(1 - \rho - \deltapid) - 1)}  \cdot \left(\frac{1}{t + T}\right).
    \end{align*}

    By Equation~\eqref{eq:equivalence-of-norms-lem-acceleration-PID-TD}, this immediately gives us a bound on the infinity norm.
    \begin{align*}
        \E\left[\norm{\fullV_t - \fullV^\pi}_\infty^2\right] 
        &\leq \frac{1}{\lpid^2}  \E\left[\snorm{\fullV_t - \fullV^\pi}^2\right]\\
        &\leq \frac{1}{\lpid^2}  \snorm{\fullV_0 - \fullV^\pi}^2\left(\frac{T}{t + T}\right)^{\epsilon(1 - \rho - \deltapid)}+ \frac{16e \epsilon^2\left(\Cpid + 2\Bpid \snorm{\fullV^\pi}^2\right)}{\lpid^4 (\epsilon(1 - \rho - \deltapid) - 1)}  \cdot \left(\frac{1}{t + T}\right)\\
        &\leq \frac{1}{\lpid^2}  \norm{\fullV_0 - \fullV^\pi}_\infty^2\left(\frac{T}{t + T}\right)^{\epsilon(1 - \rho - \deltapid)}+ \frac{16e \epsilon^2\left(\Cpid + 2\Bpid \norm{\fullV^\pi}_\infty^2\right)}{\lpid^4 (\epsilon(1 - \rho - \deltapid) - 1)}  \cdot \left(\frac{1}{t + T}\right).
    \end{align*}
    Since $\norm{\Vpid_t - V^\pi}_\infty \leq \norm{\fullV_t - \fullV^\pi}_\infty$, $\norm{\fullV^\pi}_\infty = \norm{V^\pi}_\infty$, and $\norm{\fullV_0 - \fullV^\pi}_\infty = \norm{V_0 - V^\pi}_\infty$ 
    we finally get
    \begin{align*}
        \E\left[\norm{\Vpid_t - V^\pi}_\infty^2\right] 
        &\le \frac{1}{\lpid^2}  \norm{V_0 - V^\pi}_\infty^2\left(\frac{T}{t + T}\right)^{\epsilon(1 - \rho - \deltapid)}  + \frac{16e \epsilon^2\left(\Cpid + 2\Bpid \norm{V^\pi}_\infty^2\right)}{\lpid^4 (\epsilon(1 - \rho - \deltapid) - 1)}  \cdot \left(\frac{1}{t + T}\right).
    \end{align*}
    This gives the statement of theorem by defining
    \begin{gather*}
        \cpid_2 \triangleq \frac{1}{\lpid^2}, \quad
        \cpid_3 \triangleq \frac{16e\Cpid}{\lpid^4}, \quad
        \cpid_4 \triangleq \frac{32e\Bpid}{\lpid^4}.
    \end{gather*}

\end{proof}

\subsection{Proof of Proposition~\ref{prop:error-terms-ratio}}
\begin{proof}
\begin{align*}
    \frac{\Eopttd(0)}{\Estattd(0)}
    &= \frac{ \norm{V_0 - V^\pi}_\infty^2}{\ltd^2} \cdot
    \frac{\ltd^4 (\epsilon(1 - \gamma) - 1)T}{16e \epsilon^2\left(\Ctd + 2\Btd \norm{V^\pi}_\infty^2\right)} 
\end{align*}
Due to the theorem conditions $T \ge \epsilon\ctd_1 / (1-\gamma)$ and $\epsilon \ge 2/(1-\gamma)$, which means $\epsilon(1-\gamma) - 1 \ge \epsilon(1-\gamma)/2$. We continue
\begin{align*}
    \frac{\Eopttd(0)}{\Estattd(0)}
    &\ge \frac{ \norm{V_0 - V^\pi}_\infty^2}{\ltd^2} \cdot
    \frac{\ltd^4 \epsilon(1 - \gamma)\cdot \epsilon \ctd_1}{32e \epsilon^2\left(\Ctd + 2\Btd \norm{V^\pi}_\infty^2\right)(1-\gamma)} \\
    &= \frac{\norm{V_0 - V^\pi}_\infty^2 \ltd^2  \ctd_1}{32e \left(\Ctd + 2\Btd \norm{V^\pi}_\infty^2\right)} \\
    &= \frac{\norm{V_0 - V^\pi}_\infty^2 \cdot 8 (5\gamma^2n  (1-d)  + 2) }{32e \left(n(1-d)/4 + 10\gamma^2n  (1-d) \norm{V^\pi}_\infty^2\right)} \\
    &= \frac{\norm{V_0 - V^\pi}_\infty^2 (5\gamma^2n  (1-d)  + 2)  }{e n  (1-d)\left(1 + 40\gamma^2 \norm{V^\pi}_\infty^2\right)}.
\end{align*}
Similarly for PID TD Learning, define $c =  \max((\kappa_p + \kappa_I\alpha)^2, \alpha^2)$:
\begin{align*}
    \frac{\Eoptpid(0)}{\Estatpid(0)}
    &= \frac{ \norm{V_0 - V^\pi}_\infty^2}{\lpid^2} \cdot
    \frac{\lpid^4 (\epsilon(1 - \rho) - 1)T}{16e \epsilon^2\left(\Cpid + 2\Bpid \norm{V^\pi}_\infty^2\right)}\\
    &\ge \frac{ \norm{V_0 - V^\pi}_\infty^2}{\lpid^2} \cdot
    \frac{\lpid^4 \epsilon(1 - \rho)\cdot \epsilon \cpid_1}{32e \epsilon^2\left(\Cpid + 2\Bpid \norm{V^\pi}_\infty^2\right)(1-\rho)} \\
    &= \frac{\norm{V_0 - V^\pi}_\infty^2 \lpid^2  \cpid_1}{32e \left(\Cpid + 2\Bpid \norm{V^\pi}_\infty^2\right)} \\
    &= \frac{\norm{V_0 - V^\pi}_\infty^2 \cdot 8 (15\gamma^2n c (1-d) + 2) }{32e \left(3nc(1-d)/4 + 30\gamma^2n c(1-d) \norm{V^\pi}_\infty^2\right)} \\
    &= \frac{\norm{V_0 - V^\pi}_\infty^2 (15\gamma^2nc  (1-d)  + 2)  }{3e nc  (1-d)\left(1 + 40\gamma^2 \norm{V^\pi}_\infty^2\right)}.
\end{align*}
\end{proof}


\section{Details of Gain Adaptation}
\label{sec:appendix-ga}
Tables~\ref{tab:semi-gradients-PE} and~\ref{tab:semi-gradients-Control} show the semi-gradients used for gain adaptation for Policy Evaluation and Control respectively. Algorithms~\ref{alg:PID-ALG} and~\ref{alg:PID-Q-learning-ALG} show the detailed description of the algorithm for Policy Evaluation and Control respectively.
\begin{table}[!h]
    \centering
    \caption{Semi-gradients of the Bellman residual used in the gain adaptation updates for Policy Evaluation. The learning rates are dropped to absorb them into $\eta$.}
    \bgroup
    \def\arraystretch{1.25}
    \begin{tabular}{|c|l|}
        \hline
        &  Estimated semi-gradient of $(\BRpi V_{t+1})(X_t)^2$\\
        \hline
        $\kappa_p$ & $(R_t + \gamma V_{t+1}(X'_t) - V_{t+1}(X_t)) \cdot (R_t + \gamma V_{t}(X'_t) - V_{t}(X_t))$ \\
        $\kappa_I$ & $(R_t + \gamma V_{t+1}(X'_t) - V_{t+1}(X_t)) \cdot [\beta z_{t}(X_t) + \alpha (R_t + \gamma V_{t}(X'_t) - V_{t}(X_t)) ]$\\
        $\kappa_d$ & $(R_t + \gamma V_{t+1}(X'_t) - V_{t+1}(X_t)) \cdot (V_{t}(X_t) - V_{t}'(X_t))$ \\
        \hline
    \end{tabular}
    \egroup
    \label{tab:semi-gradients-PE}
\end{table}

\begin{table}[!h]
    \centering
    \caption{Semi-gradients of the Bellman residual used in the gain adaptation updates for Control. The learning rates are dropped to absorb them into $\eta$.}
    \bgroup
    \def\arraystretch{1.25}
    \begin{tabular}{|c|l|}
        \hline
        &  Estimated semi-gradient of $(\BRstar Q_{t+1})(X_t, A_t)^2$\\
        \hline
        $\kappa_p$ & $(R_t + \gamma \max_{a \in \mathcal{A}} Q_{t+1}(X'_t, a) - Q_{t+1}(X_t, A_t)) \cdot (R_t + \gamma Q_{t}(X'_t, A'_t) - Q_{t}(X_t, A_t))$ \\
        $\kappa_I$ & $(R_t + \gamma \max_{a \in \mathcal{A}} Q_{t+1}(X'_t, a) - Q_{t+1}(X_t, A_t)) \cdot [\beta z_{t}(X_t, A_t) + \alpha (R_t + \gamma \max_{a \in \mathcal{A}} Q_{t}(X'_t, a) - Q_{t}(X_t, A_t)) ]$\\
        $\kappa_d$ & $(R_t + \gamma \max_{a \in \mathcal{A}} Q_{t+1}(X'_t, a) - Q_{t+1}(X_t, A_t)) \cdot (Q_{t}(X_t, A_t) - Q_{t}'(X_t, A_t))$ \\
        \hline
    \end{tabular}
    \egroup
    \label{tab:semi-gradients-Control}
\end{table}

\begin{algorithm}[!h]
	\caption{PID TD Learning with Gain Adaptation}
	\label{alg:PID-ALG}
	\begin{algorithmic}[1]
		\State Initialize $V_1$, $V'_1$, $z_1$, $\texttt{previous\_V}_1$, $\texttt{running\_BR}_1$, and $N_1$ to zero on all states.
		\State Initialize the gains to $\kappa_p = 1$, $\kappa_I = 0$, $\kappa_d = 0$.
		\For{$t = 1,\dots, K$}
		\State Observe state $X_t$, take action $A_t \sim \pi(\cdot \mid X_t)$, receive reward $R_t$, and observe next state $X_t'$.
        \State Set $\delta' \gets R_t + \gamma\cdot \texttt{previous\_V}_t(X_t') - \texttt{previous\_V}_t(X_t)$.
		\State Set $\delta \gets R_t + \gamma V_t(X_t') - V_t(X_t)$.
        \State Update the gains:
		\begin{align*}
		&\kappa_p \gets \kappa_p + \eta \frac{\delta\delta'}{\texttt{running\_BR}_t(X_t) + \epsilon}\\
		&\kappa_I \gets \kappa_I + \eta \frac{\delta(\beta z_t(X_t) + \alpha\delta')}{\texttt{running\_BR}_t(X_t) + \epsilon} \\
		&\kappa_d \gets \kappa_d + \eta \frac{\delta(V_t(X_t) - V'_t(X_t))}{\texttt{running\_BR}_t(X_t) + \epsilon}.
		\end{align*}
		\State Set $\texttt{update} \gets V_t(X_t) + \kappa_p \delta + \kappa_d(V_t(X_t) - V'_t(X_t)) + \kappa_I(\beta z_t(X_t) + \alpha\delta)$.
        \State Set $N_{t + 1}(X_t) \gets N_t(X_t) + 1$.
		\State Update the running values on the new states asynchronously:
		\begin{align*}
		&\texttt{running\_BR}_{t + 1}(X_t) \gets (1 - \lambda)\cdot \texttt{running\_BR}_t(X_t) + \lambda \delta^2 \\
        &\texttt{previous\_V}_{t + 1}(X_t) \gets V_t(X_t) \\
  		&V_{t + 1}(X_t) \gets (1 - \mu(N_t(X_t)))V_t(X_t) + \mu(N_t(X_t)) \cdot \texttt{update}\\
	    &V'_{t + 1}(X_t) \gets (1 - \mu(N_t(X_t)))V'_t(X_t) + \mu(N_t(X_t)) V_t(X_t) \\
		&z_{t + 1}(X_t) \gets (1 - \mu(N_t(X_t))) z_t(X_t) +\mu(N_t(X_t))(\beta z_t(X_t) + \alpha\delta).
		\end{align*}
		\EndFor
	\end{algorithmic}
\end{algorithm}

\begin{algorithm}[!h]
	\caption{PID Q-Learning with Gain Adaptation}
	\label{alg:PID-Q-learning-ALG}
	\begin{algorithmic}[1]
		\State Initialize $Q_1$, $Q'_1$, $z_1$, $\texttt{previous\_Q}_1$, $\texttt{running\_BR}_1$ to zero on all state-action pairs, and $N_1$ to zero on all states.
		\State Initialize the gains to $\kappa_p = 1$, $\kappa_I = 0$, $\kappa_d = 0$.
		\For{$t = 1,\dots, K$}
        \State Let $\pi$ be the policy derived from $Q_t$.
        \State Observe state $X_t$, take action $A_t \sim \pi(\cdot \mid X_t, A_t)$, receive reward $R_t$, and observe next state $X_t'$. Let $A_t' \gets \max_a Q_t(X_t', a)$.
        \State Set $\delta' \gets R_t + \gamma \cdot \texttt{previous\_Q}_t(X_t', A_t') - \texttt{previous\_Q}_t(X_t, A_t)$.
		\State Set $\delta \gets R_t + \gamma Q_t(X_t', A_t') - Q_t(X_t, A_t)$.
        \State Update the gains:
		\begin{align*}
		&\kappa_p \gets \kappa_p + \eta \frac{\delta\delta'}{\texttt{running\_BR}_t(X_t, A_t) + \epsilon} \\
		&\kappa_I \gets \kappa_I + \eta \frac{\delta(\beta z_t(X_t, A_t) + \alpha\delta')}{\texttt{running\_BR}_t(X_t, A_t) + \epsilon} \\
		&\kappa_d \gets \kappa_d + \eta \frac{\delta(Q_t(X_t, A_t) - Q'_t(X_t, A_t))}{\texttt{running\_BR}_t(X_t, A_t) + \epsilon}.
		\end{align*}
		\State Set $\texttt{update} \gets Q_t(X_t, A_t) + \kappa_p \delta + \kappa_d(Q_t(X_t, A_t) - Q'_t(X_t, A_t)) + \kappa_I(\beta z_t(X_t, A_t) + \alpha\delta)$.
        \State Set $N_t(X_t) \gets N_{t - 1}(X_t) + 1$.
		\State Update the running values on the new state-action pair asynchronously:
		\begin{align*}
		&\texttt{running\_BR}_{t + 1}(X_t, A_t) \gets (1 - \lambda)\cdot \texttt{running\_BR}_t(X_t, A_t) + \lambda \delta^2 \\
        &\texttt{previous\_Q}_{t + 1}(X_t, A_t) \gets Q_t(X_t, A_t) \\
        &Q_{t + 1}(X_t, A_t) \gets (1 - \mu(N_t(X_t)))Q_t(X_t, A_t) + \mu(N_t(X_t)) \cdot \texttt{update}\\
  		&Q'_{t + 1}(X_t, A_t) \gets (1 - \mu(N_t(X_t)))Q'_t(X_t, A_t) + \mu(N_t(X_t)) Q_t(X_t, A_t) \\
		&z_{t + 1}(X_t, A_t) \gets (1 - \mu(N_t(X_t))) z_t(X_t, A_t) +\mu(N_t(X_t))(\beta z_t(X_t, A_t) + \alpha\delta).
		\end{align*}
		\EndFor
	\end{algorithmic}
\end{algorithm}

\section{Description of the Environments}
\label{sec:appendix-experiment-description}

\subsection{Chain Walk}
The environment consists of 50 states that are connected in a circular chain. The agent has two actions available, moving left or right. Upon taking an action, the agent succeeds with probability 0.7, stays in place with probability 0.1, and moves in the opposite direction with probability 0.2. The agent receives a reward of 1 when entering state 10, a reward of -1 when entering state 40, and a reward of 0 otherwise.
The policy evaluated in the PE experiments is to always move left.

\subsection{Cliff Walk}
\begin{figure}[ht]
    \centering
    \includegraphics[width=0.4\textwidth]{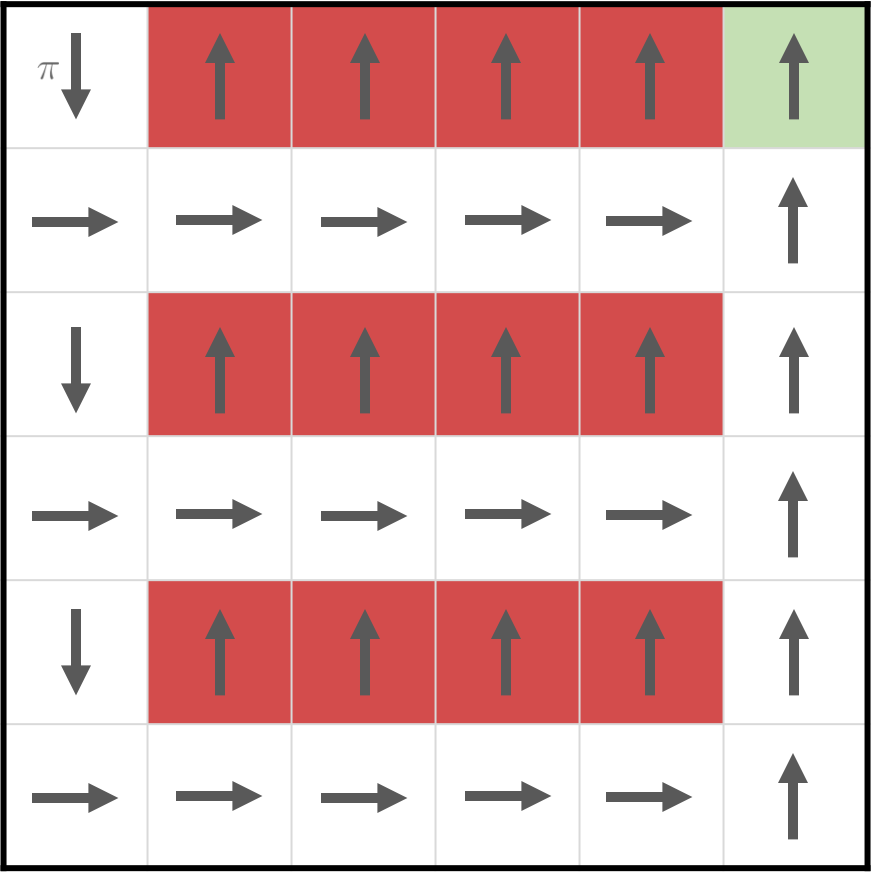}
    \caption{A visualization of Cliff Walk, taken from \cite{rakhsha2022operator}. The arrows depict the optimal policy.}
    \label{fig:gridworld}
\end{figure}
A 6 by 6 grid world is used, visualized in Figure \ref{fig:gridworld}. The agent starts on the top left. Its goal is to end up on the top right. There are 12 cliff tiles, and the agent is stuck in them if it falls in. Moreover, the agent is stuck in the goal state once entering it. Upon making a move in the goal state, it receives a reward of 20. Making a move in a cliff receives a reward of -32, -16, or -8 depending on whether the cliff is on the top, middle, or bottom respectively. Otherwise, it receives a reward of -1.
The agent has four possible actions corresponding to moving up, down, left, and right. If the agent attempts to move off the grid, it simply stays in place. Otherwise, its action succeeds with probability 0.9, and moves in one of the other three directions at random with uniform probability
The policy evaluated in the TD experiments is a random walk.

\subsection{Garnet}
The environment is randomly generated. They consist of 50 states, and 3 actions per state.
To build the environment, for each action and state, $(x, a)$, we pick 5 other random states $\XX_{x, a}$. For 10 randomly chosen states $x$, we set $r(x)$ from a uniform distribution between 0 and 1. We set $r(x)$ is zero on all other states. Then, when taking action $a$ and from state $x$, we receive reward $r(x)$ and move to any state in $\XX_{x, a}$ with equal probability.
The policy evaluated in the TD experiments is a random walk.

\section{Additional Experiments}

Figure~\ref{fig:chainwalk-adapt} shows the performance of gain adaptation on Chain Walk when $\gamma=0.99$, and the corresponding movement of the controller gains. Figure~\ref{fig:cliffwalk-adapt-control} shows the performance of gain adaptation on Cliff Walk when $\gamma=0.99$, and the corresponding movement of the controller gains.

\begin{figure}[h]
    \centering
    \begin{subfigure}{0.48\linewidth}
        \includegraphics[width=\linewidth]{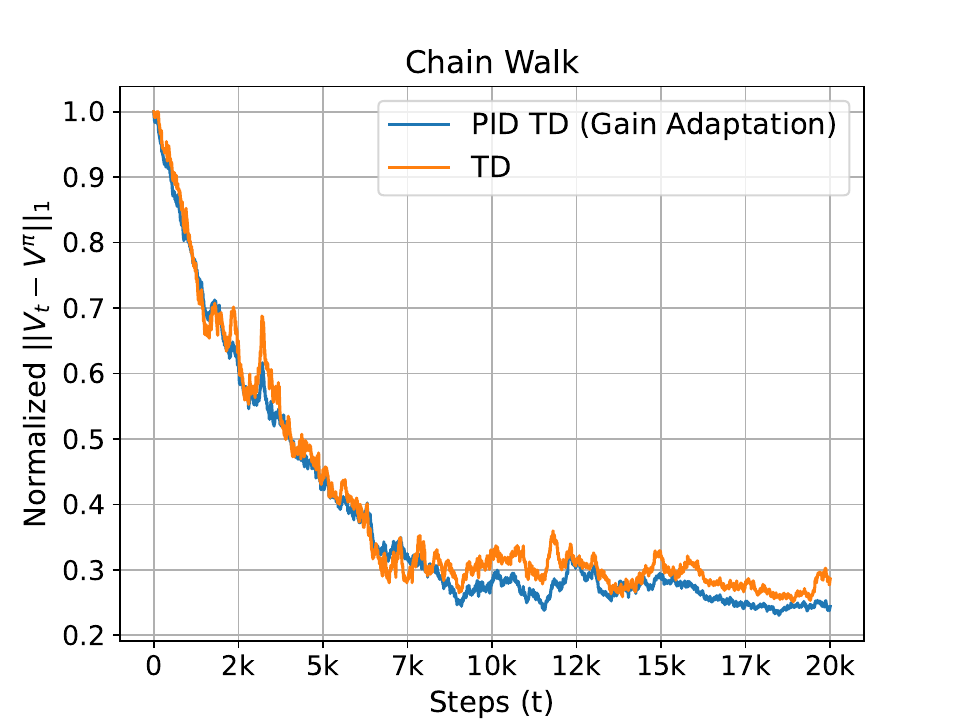}
    \end{subfigure}
    \hfill
    \begin{subfigure}{0.48\linewidth}
        \includegraphics[width=\linewidth]{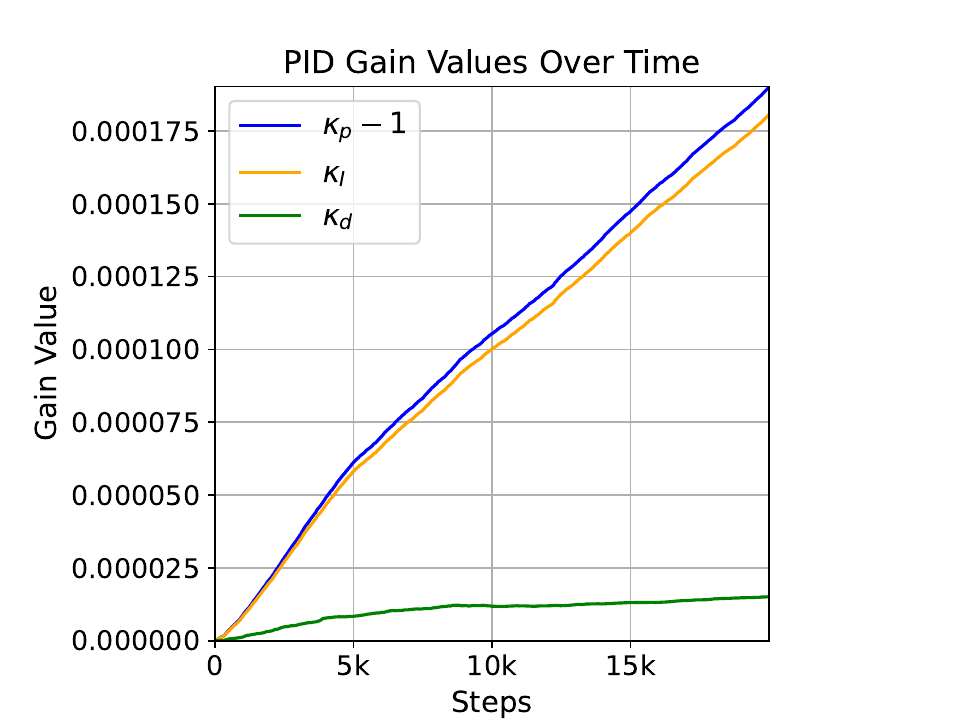}
    \end{subfigure}
    \caption{PID TD Learning with Gain Adaptation in Chain Walk with $\gamma = 0.99$. \textit{(Left)} Comparison of value errors of PID TD Learning with TD Learning. Each curve is averaged over 80 runs. Shaded area shows standard error. \textit{(Right)} The change of gains done by Gain Adaptation through training.}
    \label{fig:chainwalk-adapt}
\end{figure}

\begin{figure}[h]
    \centering
    \begin{subfigure}{0.48\linewidth}
        \includegraphics[width=\linewidth]{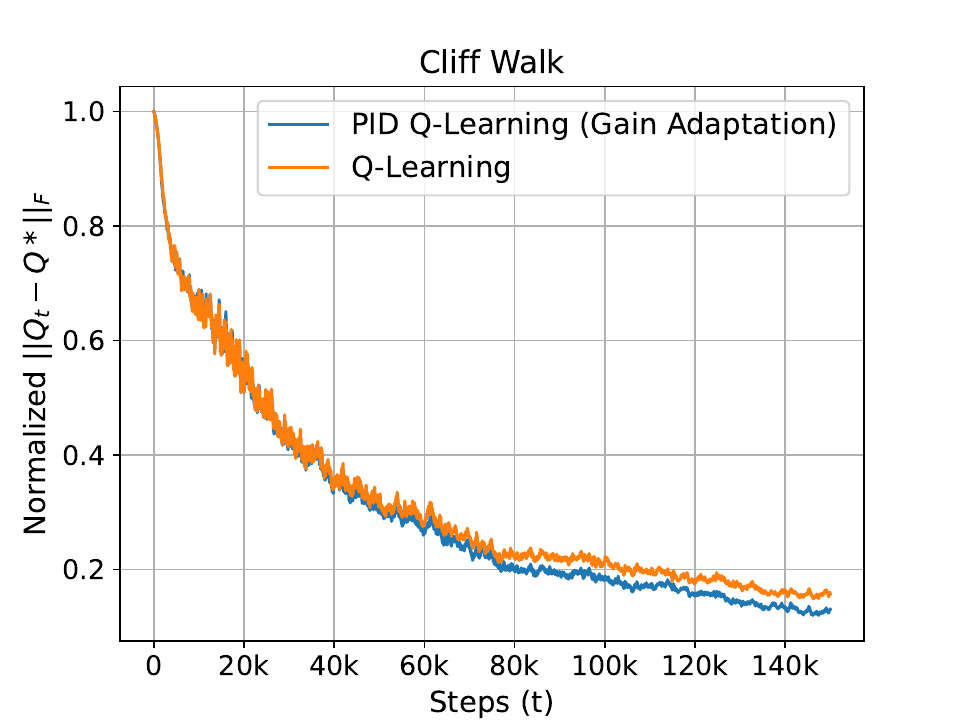}
    \end{subfigure}
    \hfill
    \begin{subfigure}{0.48\linewidth}
        \includegraphics[width=\linewidth]{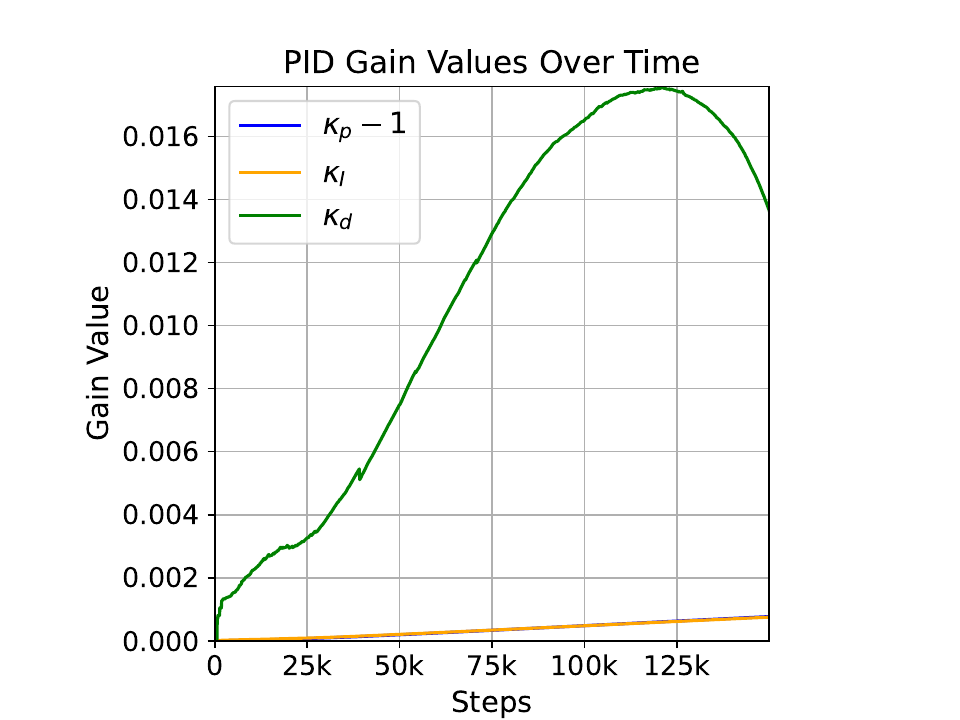}
    \end{subfigure}
    \caption{Q-Learning with Gain Adaptation in Cliff Walk with $\gamma = 0.99$. \textit{(Left)} Comparison of value errors of PID Q-Learning with Q-Learning. Each curve is averaged over 80 runs. Shaded area shows standard error. \textit{(Right)} The change of gains done by Gain Adaptation through training.}
    \label{fig:cliffwalk-adapt-control}
\end{figure}

\section{Details of Experimental Setup}
\label{sec:appendix-learning-rates}

We pick the hyperparameters such that a normalized error of $0.2$ is achieved the fastest, and if this error is not achieved, the final error is minimized.
We fix $\alpha=0.05$, $\beta=0.95$, and $\lambda=0.5$ throughout all the experiments.

For the Garnet (PE) experiments in Figure~\ref{fig:garnet-plot}, we perform a grid search on $\eta \in \{0.1, 0.01, 0.001, 0.0001\}$, $\epsilon \in \{0.1, 0.01\}$. Similarly, for the Garnet (Control) experiments, we use $\epsilon = 0.1$ and perform a grid search over $\eta \in \{10^{-5}, 5 \times 10^{-5}, 10^{-6}\}$. The learning rates we perform a grid search over in these tests is listed in Table~\ref{tab:learning-rates}. The grid search is separately performed for each instance of the sampled Garnet. For TD Learning and Q-learning, the rates searched over are the same as that of the P component in Table~\ref{tab:learning-rates}.
On each randomly generated Garnet environment, 80 runs are performed and the average trajectory is found. The variation of this average trajectory among all the 80 Garnet environments is shaded in Figure~\ref{fig:garnet-plot}.

For the Cliff Walk policy evaluation experiments in Figure~\ref{fig:cliffwalk-adapt-gains}, we set $\eta = 10^{-5}$ and $\epsilon = 10^{-1}$. For the Chain Walk (Control) experiments in Figure~\ref{fig:chainwalk-adapt-gains}, we set $\eta = 4 \times 10^{-8}$ and $\epsilon = 10^{-4}$. For the Chain Walk (PE) experiments in Figure~\ref{fig:chainwalk-adapt}, we set $\eta = 5 \times 10^{-7}$ and $\epsilon = 10^{-1}$. For the Cliff Walk (Control) experiments in Figure~\ref{fig:cliffwalk-adapt-control}, we set $\eta = 10^{-8}$ and $\epsilon = 10^{-1}$. For picking the learning rate for the I and D component, we also consider learning rates of the form $\min(0.25, N_t(X_t)/M)$ with $M \in \{\infty, 10, 100, 500, 1000, 10000\}$.

\begin{table}[ht]
    \centering
    \begin{tabular}{|c|c|}
        \hline
         & Learning Rates $\min(\epsilon, N_t(X_t)/M)$ searched through \\
         & (formatted $\epsilon: \text{corresponding set of } M$)\\
        \hline
        P Component & 
        \begin{tabular}{@{}c@{}}
            1: $\{10, 50, 100, 500, 1000, 10000\}$ \\
            0.75: $\{10, 50, 100, 500, 1000\}$ \\
            0.5: $\{10, 50, 100, 500, 1000\}$ \\
            0.25: $\{10, 50, 100\}$ \\
            0.1: $\{10, 50, 100\}$ \\
            0.01: $\{10000\}$ \\
            0.001: $\{10000\}$ \\
            0.0001: $\{10000\}$
        \end{tabular} \\
        \hline
        I Component & 
        \begin{tabular}{@{}c@{}}
            1: $\{\infty, 100\}$ \\
            0.5: $\{\infty\}$ \\
            0.1: $\{\infty\}$ \\
            0: $\{\infty\}$
        \end{tabular} \\
        \hline
        D Component & 
        \begin{tabular}{@{}c@{}}
            1: $\{\infty, 100\}$ \\
            0.5: $\{\infty\}$ \\
            0.25: $\{\infty\}$ \\
            0.1: $\{\infty\}$ \\
            0.01: $\{\infty\}$ \\
            0: $\{\infty\}$
        \end{tabular} \\
        \hline
    \end{tabular}
    \caption{All the learning rates searched through in the Garnet experiments.}
    \label{tab:learning-rates}
\end{table}

\end{document}